\documentclass[aos]{imsart}

\RequirePackage{amsthm,amsmath,amsfonts,amssymb}
\RequirePackage[numbers]{natbib}
\RequirePackage[colorlinks,citecolor=blue,urlcolor=blue]{hyperref}
\usepackage{latexsym}
\usepackage{color}
\usepackage{thm-restate}
\usepackage{tikz}
\usetikzlibrary{calc}
\usetikzlibrary{decorations.pathreplacing}
\usepackage{thmtools}
\usepackage{algorithm2e}
\PassOptionsToPackage{algo2e,ruled}{algorithm2e}
\usepackage{booktabs}
\usepackage{multirow}

\usepackage[mathscr]{eucal}

\startlocaldefs
\theoremstyle{plain}

\newtheorem{theorem}{Theorem}[section]
\newtheorem{lemma}[theorem]{Lemma}
\newtheorem{corollary}[theorem]{Corollary}
\newtheorem{proposition}[theorem]{Proposition}

\theoremstyle{remark}
\newtheorem{definition}[theorem]{Definition}

\newcommand{\ftime}{\text{F-TiME}}

\newcommand{\Acal}{\mathcal{A}}
\newcommand{\Bcal}{\mathcal{B}}
\newcommand{\Ccal}{\mathcal{C}}

\newcommand{\Ecal}{\mathcal{E}}
\newcommand{\Fcal}{\mathcal{F}}
\newcommand{\Gcal}{\mathcal{G}}
\newcommand{\Hcal}{\mathcal{H}}
\newcommand{\Ical}{\mathcal{I}}

\newcommand{\Pcal}{\mathcal{P}}
\newcommand{\Rcal}{\mathcal{R}}
\newcommand{\Scal}{\mathcal{S}}
\newcommand{\Tcal}{\mathcal{T}}
\newcommand{\Ucal}{\mathcal{U}}

\newcommand{\Xcal}{\mathcal{X}}
\newcommand{\Ycal}{\mathcal{Y}}
\newcommand{\Ocal}{\mathcal{O}}

\newcommand{\Ebb}{\mathbb{E}}
\newcommand{\Nbb}{\mathbb{N}}
\newcommand{\Pbb}{\mathbb{P}}

\newcommand{\Xbb}{\mathbb{X}}
\newcommand{\Ybb}{\mathbb{Y}}

\usepackage{bbm}

\newcommand{\1}{\mathbbm{1}}

\definecolor{dark_red}{rgb}{0.2,0,0}

\DeclareMathOperator*{\argmax}{arg\,max}

\newcommand{\comment}[1]{}

\newcommand{\mb}[1]{\ensuremath{\boldsymbol{#1}}}

\newcommand{\Event}{\mathscr{E}}

\newcommand{\EXPINF}{\mathrm{EXPINF}}
\newcommand{\EXPIX}{\mathrm{EXP3.IX}}
\newcommand{\EXP}{\mathrm{EXP3}}

\newcommand{\X}{\mathcal X}

\newcommand{\A}{\mathcal A}

\renewcommand{\P}{\mathbb P}

\newcommand{\nats}{\mathbb{N}}

\newcommand{\E}{\mathbb E}

\newcommand{\argmin}{\mathop{\rm argmin}}

\renewcommand{\limsup}{\mathop{\rm limsup}}

\DeclareSymbolFont{bbold}{U}{bbold}{m}{n}
\DeclareSymbolFontAlphabet{\mathbbold}{bbold}
\newcommand{\ind}{\mathbbold{1}}

\newcommand{\ProcX}{\mathbb{X}}

\newcommand{\target}{f^{\star}}

\newcommand{\goodpol}{\pi^{\star}}

\newcommand{\KC}{\mathcal{C}_{1}}
\newcommand{\OKC}{\mathcal{C}_{2}}
\newcommand{\SMV}{\OKC}
\newcommand{\UKC}{\mathcal{C}_{3}}
\newcommand{\FS}{\UKC}

\newcommand{\ignore}[1]{}
\newcommand{\private}[1]{}

\colorlet{sgreen}{black!45!green}
\newcommand{\sh}[1]{{\color{sgreen} #1}}
\newcommand{\moise}[1]{{\color{blue} #1}}

\endlocaldefs

\begin{document}

\begin{frontmatter}
\title {Contextual Bandits and Optimistically Universal Learning} 
\runtitle{Contextual Bandits and Optimistically Universal Learning}

\begin{aug}
\author{\fnms{Mo\"{i}se} \snm{Blanchard}\thanksref{m1}\ead[label=e1]{moiseb@mit.edu}}
\author{\fnms{Steve} \snm{Hanneke}\thanksref{m2}\ead[label=e2]{steve.hanneke@gmail.com}},
\author{\fnms{Patrick} \snm{Jaillet}\thanksref{m3}\ead[label=e3]{jaillet@mit.edu}}

\runauthor{Blanchard, Hanneke, and Jaillet}


\address{\thanksmark{m1}Massachusetts Institute of Technology, \printead{e1}}
\address{\thanksmark{m2}Purdue University, \printead{e2}}
\address{\thanksmark{m3}Massachusetts Institute of Technology, \printead{e3}}



\end{aug}

\begin{abstract}
We consider the contextual bandit problem on general action and context spaces, where the learner's rewards depend on their selected actions and an observable context. This generalizes the standard multi-armed bandit to the case where side information is available, e.g., patients' records or customers' history, which allows for personalized treatment. We focus on consistency---vanishing regret compared to the optimal policy---and show that for large classes of non-i.i.d. contexts, consistency can be achieved regardless of the time-invariant reward mechanism, a property known as \emph{universal consistency}. Precisely, we first give necessary and sufficient conditions on the context-generating process for universal consistency to be possible. Second, we show that there always exists an algorithm that guarantees universal consistency whenever this is achievable, called an \emph{optimistically universal} learning rule. Interestingly, for finite action spaces, learnable processes for universal learning are exactly the same as in the full-feedback setting of supervised learning, previously studied in the literature. In other words, learning can be performed with partial feedback without any generalization cost. The algorithms balance a trade-off between \emph{generalization} (similar to structural risk minimization) and \emph{personalization} (tailoring actions to specific contexts). Lastly, we consider the case of added continuity assumptions on rewards and show that these lead to universal consistency for significantly larger classes of data-generating processes.
\end{abstract}



\end{frontmatter}

\section{Introduction}
\label{sec:intro}

The contextual bandits setting is one of the core important problems in sequential statistical decision-making.
Abstractly, in the contextual bandit setting, a learner (or decision maker) interacts with a reward mechanism iteratively. At each iteration, the learner observes a \emph{context} (or covariate vector) $x \in \Xcal$
and selects an \emph{arm} (or action) $a\in \Acal$ to perform; it then receives a (potentially stochastic) reward depending on the context and selected action.
For example, a store may serve a sequence of customers, and for each provide a list of product recommendations, and receive reward if the recommendation leads to a purchase.
The key distinctions between the contextual bandit setting and standard supervised learning (or regression)
are that (1) the learner's objective is to obtain a near-maximum average reward over time (rather than merely estimating the reward conditional means),
and (2) the learner only observes the reward corresponding to the arm it chose.
These aspects introduce a fundamental trade-off between \emph{exploration} and \emph{exploitation}:
that is, while some arms may have high estimated reward values,
other arms may have higher uncertainty in their rewards: in particular, uncertainty about whether they would yield an even higher reward, 
so that selecting that arm may provide information about the potential for higher future rewards.

\subsection{Universal Consistency}

In the contextual bandit setting, a learner is \emph{consistent} if its average reward converges to the maximum-possible average reward obtained with an optimal policy. Naturally, one would aim for learning procedures that ensure consistency for a broad class of problem instances. In particular, we are interested in \emph{universal consistency} which asks that a learning rule achieves consistency for any underlying reward mechanism and as a by-product, any optimal policy. The equivalent notion can be defined for the full-information case: for a stream of data $(\Xbb,\Ybb)=(X_t,Y_t)_{t\geq 1}$ of instances modeled as a stochastic process on $\Xcal\times \Ycal$, a learning rule with predictions $\hat Y_t$ is consistent if it has vanishing excess error compared to any fixed measurable predictor function $f:\Xcal\to\Ycal$, i.e., $\frac{1}{T}\sum_{t=1}^T \ell(\hat Y_t,Y_t)-\ell(f(X_t),Y_t)\to 0\;(a.s.).$ Then, an algorithm is universally consistent if it is consistent irrespective of the generating process for the values $\Ybb$ from the instances $\Xbb$.
In this standard full-feedback setting, there are many works establishing universal consistency, beginning with the seminal work of \citep*{stone:77} who proved universal consistency for a broad family of \emph{local average estimators}.
Later works extended these results, to guarantee \emph{strong} universal consistency (i.e., almost sure convergence), other categories of learning rules, more general conditions on the metric space $\X$, and more general loss functions
\cite{devroye:96,gyorfi:02}. More recently, \cite{hanneke:21b,gyorfi2021universal,cohen:22} gave minimal assumptions on the space $\X$ for universal consistency---essentially-separable metric spaces. All of these works were restricted to \emph{i.i.d.} data $(X_t,Y_t)_{t\geq 1}$ sampled from a joint distribution on $\Xcal\times \Ycal$. Some of these results aimed to relax the i.i.d. assumption by considering non-i.i.d. mixing, stationary ergodic data generating processes \cite{morvai:96,gyorfi:99,gyorfi:02} or satisfying the law of large numbers \cite{morvai:99,gray:09,steinwart:09}.

\subsection{Optimistically universal learning}
In the present work we pursue a theory of universal consistency under provably-\emph{minimal} assumptions on the sequence of contexts.
This type of theory falls into a framework known as \emph{optimistically universal learning}, 
introduced by \citep*{hanneke:21}, that can be succinctly summarized as ``learning whenever learning is possible''.
The idea is to identify the minimal assumption on the data sequence sufficient for universal consistency to be possible.
Such an assumption is then both necessary and sufficient, and therefore amounts to merely assuming 
that universally consistent learning is \emph{possible}: aptly named the \emph{optimist's assumption}.
For any given process $\ProcX$ satisfying this minimal assumption, by definition there must exist a 
universally consistent learning rule. However, the interesting question becomes whether 
the optimist's assumption alone is sufficient to guarantee universal consistency for some well-designed learning rule: 
that is, whether there exists a single learning rule that is universally consistent 
for \emph{every} process $\ProcX$ satisfying the optimist's assumption.
Such a learning rule is said to be \emph{optimistically universal}.

\subsection{Optimistically universal learning with full-feedback}

The first general analysis of optimistically universal learning and provably-minimal assumptions for universal consistency in the full-feedback setting was introduced by \citep*{hanneke:21}. He provided general necessary and sufficient conditions 
for the existence of universally consistent learning rules for inductive learning---where one can only observe a finite amount of data $(X_t,Y_t)_{t\leq n}$ before committing to a prediction rule for the future steps $T\geq n$---
and for a slight variation called self-adaptive learning---where the learner only observes a finite amount of values $(X_t,Y_t)_{t\leq n}$ but can continue to update its predictions from the testing observations $X_{n+1},\ldots,X_T$, that is it continues to learn from test data. Interestingly, while there do not exist optimistically universal inductive learning rules, 
there do exist explicitly defined optimistically universal self-adaptive learning rules.
That work focused mostly on the \emph{noiseless} function learning setting where some unknown function $f^*:\Xcal\to\Ycal$ defines the values exactly via $Y_t=f(X_t)$. It also left open the question of characterizing universal learning for the standard \emph{online learning} framework, in which the learner can update its predictions from the complete available test and value data $(X_t, Y_t)_{t\leq T}$ \citep*[see also][]{hanneke:21c}.

Addressing the online learning problem, in the noiseless setting, \citep*{blanchard:22b} 
provided a simpler 
characterization and algorithm for unbounded losses while
\cite{blanchard:22a,blanchard:22c} provided a solution for the main case of interest of bounded losses. In particular, while the nearest neighbor algorithm may not be universally consistent even for i.i.d. data \cite{cerou2006nearest}, for noiseless responses, a simple variant with restricted memory is optimistically universal \cite{blanchard:22a}. For the generic case of noisy responses, \cite{hanneke:22a} showed that universal learning can be achieved even for arbitrarily dependent responses on large classes of processes. The complete characterization of universal online learning with noise was given in \cite{blanchard:22d}, showing that under mild conditions on the value space---including totally-bounded metric spaces---optimistically universal learning is possible for arbitrary or adversarial responses without generalizability cost compared to noiseless responses.

\subsection{Universal learning with partial feedback}

The contextual-bandit formulation was first introduced for one-armed bandits \cite{woodroofe1979one,sarkar1991one} in a rather restricted setting. Since then, progress has been made in the literature investigating stochastic contextual bandits under \emph{parametric} assumptions \cite{wang2005bandit,langford2007epoch, goldenshluger2009woodroofe,bubeck2012regret,auer2016algorithm,rakhlin2016bistro}. In the \emph{non-parametric} setting, significant advances have been made to obtain minimax guarantees under smoothness conditions (e.g. Lipschitz) and margin assumptions \cite{lu2009showing,rigollet2010nonparametric,slivkins2011contextual,perchet2013multi} with recent refinements including \cite{guan2018nonparametric,reeve2018k}.

However, to the best of our knowledge, there are no prior works establishing universal consistency even under all i.i.d. data sequences, i.e., consistency in the non-parametric setting without further assumptions. As such, the present work is also the first to propose such results and corresponding universally consistent learning rules. Closest to this work is the result from \cite{yang2002randomized} which shows that if rewards are continuous in the contexts, strong consistency can be achieved with familiar non-parametric methods, for Euclidean context spaces. This work significantly generalizes this result to unrestricted reward mechanisms, separable metric action and context spaces, and non-i.i.d. data.

Non-i.i.d. data has also been widely studied in the literature. Most relevant to our work are 
 non-i.i.d. generating processes for contexts. Examples include customers' profile distribution, which may change depending on seasonal patterns, or the extension of clinical trials to new populations. In these cases, the distribution of contexts $x$ changes while the underlying conditional distribution remains unchanged, a phenomenon known as \emph{covariate-shift}. Such formalism was adopted in works on domain adaptation for classification \cite{sugiyama2007direct,gretton2009covariate,ben2012hardness}. Moreover, several works have also considered distributional shifts in both contexts and responses for bandit problems, in both parametric \cite{besbes2014stochastic,luo2018efficient,wu2018learning,chen2019new} and non-parametric settings \cite{suk:21}.

\subsection{Summary of the present work}
In the present work we study optimistically universal learning in a partially-supervised setting: namely, standard contextual bandits \cite{slivkins2019introduction, lattimore2020bandit} with stationary reward functions. Precisely, there exists a time-invariant conditional probability distribution $P_{r\mid a,x}$ such that the reward $r_t$ at each iteration is sampled according to the distribution $P_{r\mid a=a_t,x=X_t}$ where $a_t$ (resp. $X_t$) denotes the selected action (resp. observed context) at time $t$, independently from the past history.
We are interested in online learning, where the learner may observe all past rewards $r_{t'}$ and contexts $X_{t'}$, $t' < t$,
when choosing its action $a_{t}$ given the context $X_t$.
We aim to achieve average reward $\frac{1}{T} \sum_{t=1}^{T} r_t$ that is (almost surely) competitive with any fixed policy $\X \to \Acal$ as $T\to\infty$.

\subsubsection{Bounded unrestricted rewards}
We first focus on the classical assumption that rewards are bounded. We show there always exists an optimistically universal learning rule.
Our approach to proving this is to first characterize which processes $\ProcX$ 
admit universally consistent learning rules, and then use this characterization to 
inform the design and analysis of a learning rule, which will be universally consistent under every such process.
However, this approach turns out to require three separate cases: 
namely, $\Acal$ finite, $\Acal$ countably infinite, and $\Acal$ uncountably infinite.
Each of these cases gives rise to a different characterization of the set of processes $\ProcX$ 
under which universally consistent learning is possible for contextual bandits, 
a fact which itself is of independent interest.
Moreover, each of these sets of processes corresponds to known families of processes 
from the past literature on optimistically universal learning.
When $\Acal$ is finite, the set of processes admitting universal consistency for contextual bandits is equivalent to the family of processes 
admitting universally consistent online learning with full supervision: a family known as $\Ccal_2$.
While this fact appears natural, interestingly this is not the case when $\Acal$ is countably infinite.
In that case, the set of processes admitting universal learning for contextual bandits is equivalent to the family of processes admitting 
universally consistent \emph{inductive} learning with full supervision: a family known as $\Ccal_1$, 
which is more restrictive than $\Ccal_2$.
Finally, when $\Acal$ is uncountably infinite,
universal learning can never be achieved.

\subsubsection{Bounded rewards under continuity assumptions}
For unrestricted rewards, although large classes of non-i.i.d. processes ($\Ccal_1$ or $\Ccal_2$) admit universal learning for countable action spaces, the answer for uncountable action spaces was very negative: universal consistency could never be achieved. However, we show that under continuity assumptions on the rewards, one can recover positive results for general action spaces. Further, in all cases, we provide optimistically universal learning rules. First, under the assumption that rewards are continuous, the characterization of processes admitting universal consistency now requires only two cases. If the action space is finite, the set of processes admitting universal learning remains unchanged and is $\Ccal_2$. On the other hand, if the action space is infinite, this set becomes $\Ccal_1$, irrespective of whether the action space was countably or uncountably infinite. Second, we consider a stronger assumption of uniform continuity on the rewards, in which the modulus of continuity of the expected reward in the actions $\bar r(\cdot,x)$ for $x\in \Xcal$ are uniform over the context space $\Xcal$. Under this assumption, universal learning under the more general set of processes $\Ccal_2$ becomes possible for a significantly larger class of action spaces, namely totally-bounded action spaces. Otherwise, universal learning is achievable exactly on $\Ccal_1$ processes.

\subsubsection{Unbounded rewards}
Last, we consider the most general case of unbounded rewards. It is known that the family of processes admitting universal consistency with full supervision and \emph{unbounded} losses is very restrictive. These are processes visiting only a finite number of distinct instances in $\Xcal$, known as $\Ccal_3$. For contextual bandits, in the standard case of unrestricted rewards, we show that there is a simple dichotomy: if the action space is countable then the set of processes admitting universal learning is still $\Ccal_3$; however, if the action space is uncountably infinite, universal learning can never be achieved. Nevertheless, under continuity assumptions on the rewards, universal learning can always be achieved under $\Ccal_3$ processes. Again, we give optimistically universal learning rules for all cases.

\subsection{Overview of probability-theoretic contributions}
In this work, we make use of the conditions $\Ccal_1$, $\Ccal_2$, and $\Ccal_3$ on stochastic processes from the universal learning literature to characterize the set of processes admitting universal learning. Along the way to establishing these results, another significant contribution of this work 
is establishing new equivalent characterizations of the families $\Ccal_1$ and $\Ccal_2$, crucial for the design of our optimistically universal algorithms.
In particular, we establish a new connection between these two families: 
proving that $\Ccal_2$ can essentially be characterized by processes that would be in $\Ccal_1$ 
if we were to replace duplicate values in the sequence $\ProcX$ by some default value $x_0$. As a result, $\Ccal_2$ processes differ from $\Ccal_1$ processes only through duplicates: if a process $\Xbb$ is guaranteed to almost never visit exactly the same context (e.g. i.i.d. processes with density) the properties $\Ccal_1$ and $\Ccal_2$ are equivalent. 
This fact has further interesting implications, such as a new technique for the design of 
optimistically universal learning rules for online learning with full supervision; prior to this, 
only one approach was known to yield such learning rules, based on a modified nearest neighbor algorithm \cite{blanchard:22a}.
The new approach suggested in the present work is instead based on an explicit model selection 
technique, in the spirit of structural risk minimization, analogous to the optimistically universal 
self-adaptive learning technique developed by \citep*{hanneke:21}.

\subsection{Overview of algorithmic techniques}
We present an overview of the optimistically universal learning rule for finite action sets, Algorithm \ref{alg:main_learning_rule}, which encompasses the main algorithmic innovation in this work. We use the property that $\Ccal_2$ processes without duplicates satisfy the $\Ccal_1$ property (Proposition \ref{prop:C2_equivalent_forms}) to separate times into two classes: points not appearing often recently and points which have many duplicates recently. 
\begin{enumerate}
    \item For the points in the first category, which behave as $\Ccal_1$ processes, we use an approach similar to structural risk minimization: we aim to achieve sublinear regret compared to a constructed countable set of policies that is empirically dense. To do so, we use a restarting technique introduced in \cite{hanneke:21}: we use classical bandit algorithms as a subroutine to achieve sublinear regret with respect to a fixed finite number of policies, and occasionally restart the bandit learner to gradually increase the number of competing policies considered.
    \item For the points in the second category, we use a completely different strategy. Intuitively, these correspond to instances with many duplicates in the recent past, hence it is advantageous to assign each frequent instance an independent bandit learner. In particular, this specific bandit learner is tailored to that point's rewards only and completely disregards historical data from other points.
\end{enumerate}

Interestingly, we can interpret the general strategy as balancing a tradeoff between \emph{generalization} and \emph{personalization}. The first strategy aims to find a policy that performs well at an aggregate level for points with few duplicates. On the other hand, the algorithm performs pure personalization for specific points that have many recent repetitions. This schematic presentation hides many details. In particular, to obtain vanishing excess error compared to the optimal policy, the algorithm needs to balance the generalization/personalization tradeoff carefully, to obtain the required generalization property. In effect, we allow for a cap $M$ of duplicates for each instance in the recent past to be treated with the generalization strategy, and adaptively increase this cap. To adaptively increase this cap, the algorithm occasionally uses ``exploration'' times to estimate the performance of each strategy, and decides to increase the cap based on these estimates. Last, in order to have decisions robust to non-stationarity in the sequence of contexts, the algorithm selects actions based on recent data: the learning procedure is broken down by periods that contain a given proportion of the past data, then this proportion adaptively decays to $0$.

\subsection{Outline of the paper}
The remainder of the paper is organized as follows. After giving the definitions and main results in Section~\ref{sec:preliminaries}, we provide in Section~\ref{sec:ingredients} new characterizations of stochastic process classes as well as base algorithms, used to construct our learning rules. With these tools, we study optimistic learning with bounded rewards for finite (Section~\ref{sec:finite-actions}), countably infinite (Section~\ref{sec:countable_actions}), and uncountable (Section~\ref{sec:uncountable_actions}) action sets. We then show that universal learning can be achieved on larger classes of processes under continuity assumptions on the rewards in Section~\ref{sec:continuity-assumptions}. Last, in Section~\ref{sec:unbounded} we consider the more restrictive case of unbounded rewards.

\newpage 
\section{Preliminaries and main results}
\label{sec:preliminaries}

\subsection{Formal setup and problem formulation}

The goal of this paper is to study the general framework of contextual bandits in an online setting. Given a separable metrizable Borel context space $(\Xcal,\Bcal)$ and a separable metrizable Borel action space $\Acal$, the learner interacts with the contextual bandit at each iteration $t\geq 1$ of the learning process in the following fashion. First, the learner observes a context $X_t\in \Xcal$, then selects an action $\hat a_t\in\Acal$ based on the past history only. As a result of the action, the learner receives a reward $r_t$. We will suppose for the most part that the rewards are bounded $r_t\in[0,\bar r] = \Rcal$ for some known $\bar r\geq 0$. Hence, except for Section \ref{sec:unbounded} in which we consider unbounded rewards, we will take without loss of generality $\bar r=1$. Crucially, the learning rule can only use the past history, which is defined formally as follows.

\begin{definition}[Learning rule]
    A \emph{learning rule} is a sequence $f_\cdot =(f_t)_{t\geq 1}$ of possibly randomized measurable functions $f_t:\Xcal^{t-1}\times \Rcal^{t-1}\times \Xcal\to \Acal$. The action selected at time $t$ by the learning rule is $\hat a_t = f_t((X_s)_{s\leq t-1},(r_s)_{s\leq t-1},X_t)$.
\end{definition}

We suppose that the contexts are generated from a stochastic process $\Xbb=(X_t)_{t \in \nats}$ on $\Xcal$. Further, we assume that rewards are sampled from a distribution conditionally on the context and actions. Formally, we assume that there exists a time-invariant conditional distribution $P_{r\mid a,x}$ such that the rewards $(r_t)_{t\geq 1}$ are conditionally independent given their respective selected action $a_t$ and observed context $x_t$, and follow this conditional distribution. Hence, $(r_t\mid a_t,x_t)_{t\geq 1}\overset{iid.}{\sim}P_{r\mid a,x}$. To emphasize the conditional dependence of $r_t$ on the actions and context, we denote $r_t(a,x)$ (resp. $r_t(a)$) the reward at time $t$, had the selected action been $a\in\Acal$ and the observed context $x\in \Xcal$ (resp. when the context at time $t$ is clear). Further, by abuse of notation, we will refer to a reward mechanism $r$ as a random variable $r\sim P_{r\mid a,x}$. For instance, we use the notation $\bar r(a,x)= \Ebb[r\mid a,x]$ to denote the immediate expected reward for any $a\in \Acal$ and $x\in \Xcal$. When we investigate unbounded rewards in Section \ref{sec:unbounded}, we will assume that the random variable $r(a, x)$ is integrable for any $(a,x)\in\Acal\times\Xcal$. We investigate three settings for the reward mechanism $r$: unrestricted, continuous, and uniformly-continuous. For the two last settings, we suppose that $\Acal$ is a separable metric space with metric $d$. We formally define the two continuity assumptions below.
\begin{definition}\label{def:continuous+unif_cont_rewards}
    The reward mechanism $r$ is continuous if for any $x\in\Xcal$, the immediate expected reward function $\bar r(\cdot, x):\Acal\to[0,1]$ is continuous.
    
    \noindent The reward mechanism $r$ is uniformly-continuous if for any $\epsilon>0$ there exists $\Delta(\epsilon)>0$ with
\begin{equation*}
    \forall x \in\Xcal,\forall a,a'\in \Acal,
    \quad d(a,a') \leq \Delta(\epsilon)\Rightarrow |\bar r(a, x)-\bar r(a', x) |\leq \epsilon.
\end{equation*}
\end{definition}

Our goal is to design algorithms that intuitively converge to the optimal policy $\pi^*:\Xcal\to\Acal$ that selects for any context $x\in\Xcal$ an optimal arm in $\argmax_{a\in\Acal} \bar r(a, x)$.  Such an \emph{optimal} policy $\pi^*$ is well-defined for finite $\Acal$; however, for infinite $\Acal$, this may no longer exist (e.g., if $\sup_{a\in \Acal} \bar{r}(a, x)$ is not attained).  Thus, to be fully general, we 
instead ask that the regret of the algorithm be sublinear compared to \emph{any} fixed measurable policy $\pi^*:\Xcal\to\Acal$. We are then interested in learning rules that are consistent irrespective of the unknown reward mechanism $r$, i.e., which intuitively converge to the (near-)optimal policy for all reward mechanisms. We follow the definitions from the universal learning literature for general processes as introduced in \cite{hanneke:21}.

\begin{definition}[Consistence and universal consistency]
    Let $\Xbb$ be a stochastic process on $\Xcal$, $r$ be a reward mechanism, and $f_\cdot$ be a learning rule. Denote by $(\hat a_t)_{t\geq 1}$ its selected actions. We say that $f_\cdot$ is \emph{consistent} under $\Xbb$ with rewards $r$ if for any measurable policy $\pi^*:\Xcal\to\Acal$,
    \begin{equation*}
        \limsup_{T\to\infty}\frac{1}{T}\sum_{t=1}^T r_t(\pi^*(X_t)) - r_t(\hat a_t) \leq 0,\quad (a.s.).
    \end{equation*}
    We say that a learning rule is \emph{universally consistent} if it is consistent under $\Xbb$ for any reward mechanism $r$.
\end{definition}

Unfortunately, universal consistency is not always achievable. For example, on $\Xcal=\Nbb$, under the process $\Xbb=(t)_{t\geq 1}$, there does not exist any universally consistent learning rule even in the simplest framework of noiseless---realizable---online learning with full-feedback---when one observes not only the reward $r_t(\hat a_t)$ but the complete vector $(r_t(a))_{a\in\Acal}$ at step $t$ \cite{hanneke:21,blanchard:22a}. Two natural questions then arise. First, when is universal consistency possible? And second, which algorithms are universally consistent for a large family of such stochastic processes? To this end, we introduce the notion of optimistically universal learning rules, that ``learn whenever learning is possible''.

\begin{definition}[Optimistically universal learning rule]
    Denote by $\Ccal $ the set of processes $\Xbb$ on $\Xcal$ such that there exists a learning rule universally consistent under $\Xbb$.
    
    We say that a learning rule $f_\cdot$ is \emph{optimistically universal} if it is universally consistent under any process $\Xbb\in\Ccal $.
\end{definition}
Similarly, we define $\Ccal^c$ (resp. $\Ccal^{uc}$) the set of processes admitting universal learning under continuous (resp. uniformly-continuous) rewards, and define accordingly the notion of optimistically universal learning rule for continuous (resp. uniformly-continuous) rewards. In this paper, we answer the informal questions described above by 1. characterizing the set of learnable processes and 2. showing that there indeed exists and providing optimistically universal learning rules.

\subsection{Useful classes of stochastic processes}

In this subsection we present the key conditions arising in the characterizations of processes on $\Xcal$ admitting universal learning. Let us first start with some notation. For any stochastic process $\Xbb=(X_t)_{t\geq 1}$, we denote $\Xbb_{\leq t} = (X_s)_{s\leq t}$ for any $t\geq 1$. We also introduce the empirical limsup frequency $\hat\mu_{\Xbb}$ as follows,
\begin{equation*}
    \hat\mu_{\Xbb}(A) = \limsup_{T\to\infty} \sum_{t=1}^T \1_A(X_t),\quad A\in\Bcal.
\end{equation*}
The first condition asks that the set function $\Ebb[\hat\mu_\Xbb(\cdot)]$ forms a \emph{continuous sub-measure}.

\begin{definition}[Condition 1 \citep*{hanneke:21}]
  \label{con:kc}
  For every monotone sequence $\{A_k\}_{k=1}^{\infty}$ of measurable subsets of $\X$ with $A_k \downarrow \emptyset$,
  \begin{equation*}
    \lim\limits_{k \to \infty} \E\!\left[ \hat{\mu}_{\ProcX}(A_k) \right] = 0.
  \end{equation*}
  We define $\Ccal_1$ as the set of processes $\Xbb$ satisfying this condition.
\end{definition}

For our purposes, we will need to extend this definition to extended stochastic processes which may take values on a subset of possibly random times $\Tcal\subset \Nbb$ instead of the complete set of times $\Nbb$. Overloading the notation $\Ccal_1$, we refer to the same condition $\Ccal_1$ for extended stochastic processes which satisfy the equivalent condition.

\begin{definition}[Extended condition 1]\label{def:extended_C1}
Given a possibly random set of times $\Tcal\subset\Nbb$, $\tilde\Xbb=(X_t)_{t\in\Tcal}$ satisfies the following condition: for every monotone sequence of measurable sets of $\Xcal$ with $A_k\downarrow\emptyset$,
\begin{equation*}
    \lim_{k\to\infty}\Ebb\left[\limsup_{T\to\infty} \frac{1}{T}\sum_{t\leq T, t\in\Tcal} \1_{A_k}(X_t)\right] = 0.
\end{equation*}

\end{definition}

As an important remark, the set of $\Ccal_1$ extended stochastic processes is larger than the processes $\tilde \Xbb = (X_t)_{t\in\Tcal}$ satisfying $(X_{t_i})_{i\geq 1}\in\Ccal_1$ where $\Tcal=\{t_1 \leq t_2\leq \ldots\}$ is an enumeration of $\Tcal$. For instance, on $\Xcal=\Nbb$, the process $(X_t=t)_{t\geq 1}$ does not belong to $\Ccal_1$---the decreasing sequence $A_k = \{n\geq k\}$ disproves the condition. However, for any increasing sequence of times $t_k=\omega(k)$, the extended process $\tilde \Xbb = (X_t)_{t\in\{t_k,k\geq 1\}}$ with $X_{t_k}=k$ for all $k\geq 1$, belongs to $\Ccal_1$ because $|\{t_k\leq T,k\geq 1\}|=o(T)$.

We then introduce a weaker condition on stochastic processes which asks that the process visits a sublinear number of sets from any measurable partition of $\Xcal$.

\begin{definition}[Condition 2 \citep*{hanneke:21}]
  \label{con:smv}
  For every sequence $\{A_k\}_{k=1}^{\infty}$ of disjoint measurable subsets of $\X$,
  \begin{equation*}
    |\{ k : \Xbb_{\leq T} \cap A_k \neq \emptyset \}| = o(T) \text{ (a.s.)}.
  \end{equation*}
  Denote by $\SMV$ the set of all processes 
  $\ProcX$ satisfying this condition.
\end{definition}

It is known \cite{hanneke:21} that $\Ccal_1\subset\Ccal_2$ and that i.i.d. processes, stationary ergodic processes, stationary processes and processes satisfying the law of large numbers---for any $A\in\Bcal$, the limit $\lim_{T\to\infty}\frac{1}{T}\sum_{t=1}^T\1_A(X_t)$ exists almost surely---belong to $\Ccal_1$. Therefore, both $\Ccal_1$ and $\Ccal_2$ are very general classes of processes.

Last, we introduce a significantly stronger assumption asking that the process only visits a finite number of distinct points.

\begin{definition}[Condition 3 \citep*{hanneke:21,blanchard:22b}]
  \label{con:fs}
  \begin{equation*} 
  |\{ x : \ProcX \cap \{x\} \neq \emptyset \}| < \infty \text{ (a.s.)}.
  \end{equation*}
  Denote by $\FS$ the set of all processes 
  $\ProcX$ satisfying this condition.
\end{definition}

\subsection{Main results}

We are now ready to present our main results. We show that the set of processes admitting universal learning $\Ccal$ corresponds to one of the classes of processes $\Ccal_3\subset\Ccal_1\subset\Ccal_2$ and depends \emph{only} on the action set $\Acal$. A summary of the charcaterizations is provided in Table \ref{table:summary_of_results}. In addition, we always provide optimistically universal learning rules for each case, which we construct in the next sections. In the main setting of bounded rewards, the relevant alternatives are whether $\Acal$ is finite, countably infinite, or uncountable.

\begin{table}
\caption{Characterization of learnable instance processes for universal learning in contextual bandits depending on properties of the action space $\Acal$.}
\label{table:summary_of_results}

\begin{tabular}{|c|c |c |c|} 
 \hline
 \multirow{4}{*}{ \textbf{Bounded rewards} } & \textbf{Unrestricted rewards} & \textbf{Continuous rewards} & $\begin{array}{c}
    \textbf{Uniformly-continuous} \\
     \textbf{rewards}
 \end{array}$\\ [0.5ex] 
 \cline{2-4} 
 &
 $\begin{array}{lc}
      \text{Finite:}& \Ccal_2  \\
      \text{Countably infinite:}& \Ccal_1  \\
      \text{Uncountable:}& \emptyset
 \end{array}$
 &$\begin{array}{l p{0.3cm} c}
      \text{Finite:}&& \Ccal_2  \\
      \text{Infinite:}&& \Ccal_1
 \end{array}$
 &$\begin{array}{lc}
      \text{Totally-bounded:}& \Ccal_2  \\
      \text{Non-totally-bounded:}& \Ccal_1
 \end{array}$ \\ 
 \hline \hline
  \textbf{Unbounded rewards} & $\begin{array}{l p{0.6cm} c}
      \text{Countable:}&& \Ccal_3  \\
      \text{Uncountable:}&& \emptyset
 \end{array}$
 &$\Ccal_3$
 &$\Ccal_3$ \\ 
 \hline
 
\end{tabular}

\end{table}

\begin{theorem}[Unrestricted bounded rewards]\label{thm:unrestricted_rewards}
    Let $\Xcal$ be a separable metrizable Borel context space and $\Acal$ an action space.
    \begin{itemize}
        \item If $\Acal$ is finite and $|\Acal|\geq 2$, then $\Ccal = \Ccal_2$.
        \item If $\Acal$ is infinite and countable, then $\Ccal= \Ccal_1$.
        \item If $\Acal$ is an uncountable separable metrizable Borel space, then $\Ccal = \emptyset$.
    \end{itemize}
    In all cases there is an optimistically universal learning rule.
\end{theorem}

We recall that $\Xbb\in \Ccal_2$ is necessary to achieve universal learning under $\Xbb$ even in the simplest online learning setting with full-feedback and noiseless values \cite{hanneke:21,blanchard:22a}. Therefore, Therorem \ref{thm:unrestricted_rewards} shows that universal consistence for contextual bandits is achievable for finite action sets at no extra generalizability cost. Unfortunately, in uncountable action spaces, universal consistence is not achievable. A natural question then becomes whether with additional mild assumptions on the rewards one can recover the large classes of processes $\Ccal_1$ or $\Ccal_2$ for universal learning. In particular, we assume that $(\Acal,d)$ is a separable metric space and first consider the case of \emph{continuous} rewards. Under this first assumption, we show that we can achieve universal consistency on all $\Ccal_1$ processes with an optimistically universal learning rule.

\begin{theorem}[Continuous bounded rewards]\label{thm:continuous_rewards}
    Let $\Xcal$ be a separable metrizable Borel context space and $(\Acal,d)$ a separable metric action space.
    \begin{itemize}
        \item If $\Acal$ is finite and $|\Acal|\geq 2$, then $\Ccal^c = \Ccal_2$.
        \item If $\Acal$ is infinite, then $\Ccal^c = \Ccal_1$.
    \end{itemize}
    In all cases there is an optimistically universal learning rule for continuous rewards.
\end{theorem}

As a result, under the the continuity assumption, one recovers the set of processes $\Ccal_1$ for infinite action spaces. However, it is not sufficient to recover the largest set $\Ccal_2$ which is necessary even in the noiseless full-feedback setting. To this ends, we consider the stronger assumption that rewards are uniformly-continuous and show that one can to recover the set of learnable processes $\Ccal_2$ for totally-bounded action spaces.

\begin{theorem}[Uniformly-continuous bounded rewards]\label{thm:uniformly_continuous_rewards}
    Let $\Xcal$ be a separable metrizable Borel context space and $(\Acal,d)$ a separable metric action space.
    \begin{itemize}
        \item If $\Acal$ is totally-bounded and $|\Acal|\geq 2$, then $\Ccal^{uc} = \Ccal_2$.
        \item If $\Acal$ is non-totally-bounded, then $\Ccal^{uc} = \Ccal_1$.
    \end{itemize}
    In all cases there is an optimistically universal learning rule for uniformly-continuous rewards.
\end{theorem}

Last, we investigate the more restrictive case of unbounded rewards in $\Rcal=[0,\infty)$. \cite{blanchard:22c} showed that even in the simplest noiseless and full-feedback online learning framework, for unbounded rewards, $\Ccal_3$ is necessary for universal learning. We show that although it forms a restrictive class of processes, universal learning under $\Ccal_3$ processes is still possible for contextual bandits. However, continuity or uniform continuity assumptions are not sufficient to enlarge this set of learnable processes.

\begin{theorem}[Unbounded rewards]\label{thm:unbounded}
    Let $\Xcal$ be a separable metrizable Borel context space and $(\Acal,d)$ a separable metric action space.
    \begin{itemize}
        \item  If $\Acal$ is countable, and $|\Acal|\geq 2$, then $\Ccal = \Ccal_3$. If $\Acal$ is uncountable, then $\Ccal=\emptyset$.
        \item $\Ccal^c=\Ccal^{uc} = \Ccal_3$.
    \end{itemize}
    In all cases there is an optimistically universal learning rule for the corresponding rewards (unrestricted, continuous or  uniformly-continuous).
\end{theorem}

\section{Base ingredients for the proofs and algorithms}
\label{sec:ingredients}

\subsection{Equivalent characterizations of stochastic process classes}
We give new characterizations of the classes $\Ccal_1$ and $\Ccal_2$ of independent interest.

We first show that for processes $\Xbb\notin\Ccal_1$, we can construct a measurable partition visited linearly by the process up to a known maximum number of duplicates in the instances for each set of the partition. This also characterizes $\Ccal_1$.

\begin{lemma}
\label{lem:infrequent-cells}
For any $\ProcX \notin \KC$, there exists a 
disjoint sequence $\{B_i\}_{i=1}^{\infty}$ of measurable subsets of $\X$ 
with $\bigcup\limits_{i \in \nats} B_i = \X$, 
and a sequence $N_i$ in $\nats$ 
such that, letting $i_t$ be the unique $i \in \nats$ 
with $X_t \in B_i$, 
with probability strictly greater than zero, it holds that
\begin{equation*}
\limsup_{T \to \infty} \frac{1}{T} \sum_{t=1}^{T} \ind\!\left[ |\Xbb_{<t} \cap B_{i_t}| < N_{i_t} \right] > 0.    
\end{equation*}
In fact, $\ProcX \notin \KC$ if and only if this holds.
\end{lemma}

\begin{proof}
Suppose $\ProcX \notin \KC$.
By Lemma 14 of \citep*{hanneke:21}, 
there exists a disjoint sequence $\{B_i\}_{i=1}^{\infty}$ 
of measurable subsets of $\X$  
such that, on an event $\Event_0$ of probability strictly great than $0$, 
it holds that 
\begin{equation*}
\lim_{j \to \infty} \hat{\mu}_{\ProcX}\!\left( \bigcup_{i \geq j} B_i \right) > 0.
\end{equation*}
Without loss of generality, we may suppose 
$B_1 = \X \setminus \bigcup\limits_{i > 1} B_i$ 
so that $\bigcup\limits_{i \in \nats} B_i = \X$.
Define a random variable $\alpha$ as 
\begin{equation*}
\alpha = \lim_{j \to \infty} \hat{\mu}_{\ProcX}\!\left( \bigcup_{i \geq j} B_i \right).
\end{equation*}

Inductively define sequences $T_k$, $J_k$ in $\nats$ as follows.
Let $T_0 = 0$ and $J_0 = 1$.
For each $k \in \nats$, 
suppose $T_{k-1}$ and $J_{k-1}$ are defined, elements of $\nats$,
and define $T_k$ and $J_k$ as follows.
Note that, by definition of $\hat{\mu}_{\ProcX}$, 
there exists an $\ProcX$-dependent random variable 
$\tau_k \in \nats$ with $\tau_k > T_{k-1}$
such that 
\begin{equation*}
\frac{1}{\tau_k} \left| \Xbb_{\leq \tau_k} \cap \bigcup_{i \geq J_{k-1}} B_i \right| 
\geq (1/2)\hat{\mu}_{\ProcX}\!\left( \bigcup_{i \geq J_{k-1}} B_i \right).
\end{equation*}
Moreover, by monotonicity of $\hat{\mu}_{\ProcX}(\cdot)$, the 
right hand side is no smaller than $\alpha/2$.
Let $T_{k} \in \nats$ be any finite non-random value such that 
\begin{equation*}
\P\!\left( \tau_k > T_k \right) < \P(\Event_0) 2^{-k-2}.
\end{equation*}
Next note that, since the sets $B_i$ are disjoint, 
there exists a finite $\ProcX$-dependent random variable $j_k \in \nats$ 
with $j_k > J_{k-1}$ such that 
\begin{equation*}
\Xbb_{\leq T_{k}} \cap \bigcup_{i \geq j_k} B_i = \emptyset. 
\end{equation*}
Let $J_k \in \nats$ be any finite non-random value such that
\begin{equation*}
\P\!\left( j_k > J_k \right) < \P(\Event_0) 2^{-k-2}.
\end{equation*}
In particular, on the event that $j_k \leq J_k$, it holds that
\begin{equation*}
\Xbb_{\leq T_{k}} \cap \bigcup_{i \geq J_k} B_i = \emptyset, 
\end{equation*}
which implies that 
\begin{equation*}
\Xbb_{\leq T_{k}} \cap \bigcup_{i \geq J_{k-1}} B_i 
= \Xbb_{\leq T_{k}} \cap \bigcup_{J_{k-1} \leq i < J_{k}} B_i.
\end{equation*}
Thus, if both events $\tau_k \leq T_k$ and $j_k \leq J_k$ hold, 
it must be that 
\begin{equation*}
\frac{1}{\tau_k} \left| \Xbb_{\leq \tau_k} \cap \bigcup_{J_{k-1} \leq i < J_{k}} B_i \right| 
\geq \alpha/2,
\end{equation*}
or equivalently, 
\begin{equation}
\label{eqn:tauk-Jk-inclusion}
\frac{1}{\tau_k} \sum_{t=1}^{\tau_k} \ind[ i_t \in \{J_{k-1} \leq i < J_{k}\} ] \geq \alpha/2.
\end{equation}
This completes the inductive definition of the sequences $T_k$ and $J_k$.

To specify the $N_i$ values, for each $k \in \nats$ and 
$i \in \{J_{k-1},\ldots,J_{k}-1\}$, define $N_i = T_k$.
Note that the event 
$\Event_1 = \Event_0 \cap \bigcap\limits_{k \in \nats} \{ \tau_k \leq T_k \} \cap \{ j_k \leq J_k \}$ 
has probability at least 
\begin{equation*} 
\P(\Event_0) - \sum\limits_{k \in \nats} \P(\Event_0) 2^{-k-1} = \P(\Event_0)/2 > 0
\end{equation*}
by the union bound.
On the event $\Event_1$, \eqref{eqn:tauk-Jk-inclusion} holds 
for every $k \in \nats$. 
Since $\tau_k \leq T_k$ on $\Event_1$, 
we also trivially have that every $i \in \{J_{k-1},\ldots,J_{k}-1\}$ 
and $t \in [\tau_k]$ satisfy $| \Xbb_{< t} \cap B_i | < t \leq T_k = N_i$.
Together, these facts imply that on $\Event_1$, every $k \in \nats$ 
satisfies 
\begin{equation*}
\frac{1}{\tau_k} \sum_{t=1}^{\tau_k} \ind\!\left[ \left| \Xbb_{< t} \cap B_{i_t} \right| < N_{i_t}  \right] \geq \alpha/2.
\end{equation*}
Since we also have $\alpha > 0$ on the event $\Event_1$,
and since $T_k$ is strictly increasing, and $\tau_k > T_{k-1}$ 
implies $\tau_k \to \infty$ as $k \to \infty$, 
altogether we have that on the event $\Event_1$, 
\begin{align*}
& \limsup_{T \to \infty} \frac{1}{T} \sum_{t=1}^{T} \ind\!\left[ \left| \Xbb_{< t} \cap B_{i_t} \right| < N_{i_t} \right]
\\ & \geq \limsup_{k \to \infty} \frac{1}{\tau_k} \sum_{t=1}^{\tau_k} \ind\!\left[ \left| \Xbb_{< t} \cap B_{i_t} \right| < N_{i_t} \right] 
\geq \alpha/2 > 0.
\end{align*}

We establish the final claim that such a result is not possible for $\ProcX \in \KC$, as follows.
Fix any $\ProcX \in \KC$.
For any disjoint sequence $B_i$ of measurable 
subsets of $\X$, and any sequence $N_i \in \nats$,
define $C_n = \bigcup \{ B_i : N_i > n \}$, 
and note that $C_n \downarrow \emptyset$.
For every $n \in \nats$, we have
\begin{align}
& \limsup_{T \to \infty} \frac{1}{T} \sum_{t=1}^{T} \ind\!\left[ \left| \Xbb_{< t} \cap B_{i_t} \right| < N_{i_t} \right] \label{eqn:KC-Ni-freq}
\\ & \leq \limsup_{T \to \infty} \frac{1}{T} \sum_{t=1}^{T} \left( \ind\!\left[ \left| \Xbb_{< t} \cap B_{i_t} \right| < n \right] + \ind\!\left[ N_{i_t} > n \right] \right) \notag 
\\ & \leq \left(  \limsup_{T \to \infty} \frac{1}{T} \sum_{t=1}^{T} \ind\!\left[ \left| \Xbb_{< t} \cap B_{i_t} \right| < n \right] \right) + \hat{\mu}_{\ProcX}\!\left( C_n \right). \notag
\end{align}
For any $m \in \nats$, any $t \geq m$ has 
$\ind\!\left[ \left| \Xbb_{< t} \cap B_{i_t} \right| < n \right] \leq \ind\!\left[ \left| \Xbb_{< m} \cap B_{i_t} \right| < n \right]$, 
so that 
\begin{align*}
& \limsup_{T \to \infty} \frac{1}{T} \sum_{t=1}^{T} \ind\!\left[ \left| \Xbb_{< t} \cap B_{i_t} \right| < n \right]
\\ & \leq \limsup_{T \to \infty} \frac{m}{T} + \frac{1}{T} \sum_{t=1}^{T} \ind\!\left[ \left| \Xbb_{< m} \cap B_{i_t} \right| < n \right]
= \hat{\mu}_{\ProcX}\!\left( \bigcup \{ B_i : |\Xbb_{< m} \cap B_i| < n \}   \right).
\end{align*}
Since the first expression above has no dependence on $m$, 
the conclusion remains valid in the limit of $m \to \infty$, 
so that 
\begin{equation*}
\limsup_{T \to \infty} \frac{1}{T} \sum_{t=1}^{T} \ind\!\left[ \left| \Xbb_{< t} \cap B_{i_t} \right| < n \right]
 \leq \lim_{m \to \infty} \hat{\mu}_{\ProcX}\!\left( \bigcup \{ B_i : |\Xbb_{< m} \cap B_i| < n \} \right),
\end{equation*}
which equals zero almost surely (by Lemmas 13 and 14 of \citealp*{hanneke:21}).
Altogether, for any $n \in \nats$, 
with probability one, 
\eqref{eqn:KC-Ni-freq} is at most
$\hat{\mu}_{\ProcX}\!\left( C_n \right)$.
Again, since \eqref{eqn:KC-Ni-freq} has no dependence on $n$, 
this inequality remains valid in the limit as $n \to \infty$, 
so that with probability one, 
\eqref{eqn:KC-Ni-freq} is at most 
\begin{equation*}
\lim\limits_{n \to \infty} \hat{\mu}_{\ProcX}\!\left( C_n \right),
\end{equation*}
which equals zero almost surely 
(by Lemma 13 of \citealp*{hanneke:21}).
The conclusion that \eqref{eqn:KC-Ni-freq} equals zero almost surely 
follows by the union bound.
\end{proof}

Next, we give a new characterization of $\Ccal_2$ processes, which also provides motivation for our generalization of $\Ccal_1$ to extended processes in Definition \ref{def:extended_C1}. This extension will be essential in our algorithms.

\begin{proposition}\label{prop:C2_equivalent_forms}
Let $\Xbb$ be a stochastic process on $\Xcal$, and define for any $M\geq 1$,
\begin{equation*}
    \Tcal^{\leq M} = \left\{t\geq 1: \sum_{t'\leq t}\1[X_{t'}=X_t]\leq M \right\},
\end{equation*}
the set of times which are duplicates of index at most $M$. In particular, $\Tcal^{\leq 1}$ is the set of times where we delete all duplicates. The following are equivalent.
\begin{enumerate}
    \item $\Xbb\in\Ccal_2$.
    \item $(X_t)_{t\in\Tcal^{\leq 1}} \in\Ccal_1$.
    \item For all $M\geq 1$, $(X_t)_{t\in\Tcal^{\leq M}} \in\Ccal_1$.
\end{enumerate}
\end{proposition}

Essentially, the main difference between extended $\Ccal_1$ and $\Ccal_2$ processes lies in the multiple occurrences of instance points. In particular, if $\Xbb$ never visits the same instance point twice almost surely, as is the case of i.i.d. process with densities, then $\Xbb\in\Ccal_1$ if and only if $\Xbb\in\Ccal_2$.

\begin{proof}
We start by showing $(2)\Rightarrow (1)$. Suppose that a process $\Xbb$ is not in $\Ccal_2$. We aim to show that $\Xbb$ disproves the second property. Because $\Xbb\notin\Ccal_2$, there exists a sequence of disjoint measurable sets $(B_i)_{i\geq 1}$, $\epsilon, \delta>0$ such that with probability $\delta>0$
\begin{equation*}
    \limsup_{T\to\infty} \frac{|\{i: \Xbb_{\leq T}\cap B_i\neq\emptyset\}|}{T} \geq \epsilon.
\end{equation*}
Denote by $\Acal$ this event, and consider the sets $A_i = \bigcup_{j\geq i} B_i$ for $i\geq 1$. Now fix $i\geq 1$. For any $T\geq 1$, we have
\begin{equation*}
    \sum_{t\leq T,t\in\Tcal^{\leq 1}}\1_{A_i}(X_t)=|A_i\cap\Xbb_{\leq T}|\geq |\{j\geq i: B_j\cap\Xbb_{\leq T}\neq \emptyset\}|\geq |\{j: \Xbb_{\leq T}\cap B_j\neq\emptyset\}| - (i-1),
\end{equation*}
where in the first inequality we used the fact that the $B_j$ are disjoint for all $j\geq i$, but included within $A_i$. As a result, on the event $\Acal$ we have $\limsup_{T\to\infty} \frac{1}{T}\sum_{t\leq T,t\in\Tcal^{\leq 1}}\1_{A_i}(X_t) \geq \epsilon.$ Hence, 
\begin{equation*}
    \Ebb \left[\limsup_{T\to\infty} \frac{1}{T}\sum_{t\leq T,t\in\Tcal^{\leq 1}}\1_{A_i}(X_t) \right] \geq \epsilon \Pbb[\Acal] = \epsilon\delta.
\end{equation*}
This holds for all $i\geq 1$ but $A_i\downarrow\emptyset$, which shows that $\Xbb$ does not satisfy property (2).

To prove $(1)\Rightarrow (2)$, now suppose that property (2) is not satisfied by $\Xbb$. We aim to show that $\Xbb\notin\Ccal_2$. Then, there exists a sequence of measurable sets $A_i\downarrow\emptyset$, $\epsilon>0$ and an increasing sequence of indices $(i_k)_{k\geq 1}$ such that for all $k\geq 1$
\begin{equation*}
     \Ebb \left[\limsup_{T\to\infty} \frac{|A_{i_k}\cap\Xbb_{\leq T}|}{T} \right] \geq \epsilon.
\end{equation*}
Because the sets $A_i$ are decreasing and the quantity within the expectation is increasing in the set $A$, this shows that for all $i\geq 1$, we have $\Ebb \left[\limsup_{T\to\infty} \frac{|A_i\cap\Xbb_{\leq T}|}{T} \right] \geq \epsilon.$ Therefore, for any $i\geq 1$ because $\Ebb\left[\limsup_{T\to\infty} \frac{|A_i\cap\Xbb_{\leq T}|}{T} \right] \leq \Pbb\left[\limsup_{T\to\infty} \frac{|A_i\cap\Xbb_{\leq T}|}{T} \geq \frac{\epsilon}{2}\right] + \frac{\epsilon}{2}$ we obtain for all $i\geq 1$
\begin{equation*}
    \Pbb\left[\limsup_{T\to\infty} \frac{|A_i\cap\Xbb_{\leq T}|}{T} \geq \frac{\epsilon}{2} \right] \geq \frac{\epsilon}{2}.
\end{equation*}
Again, because the inner quantity is increasing in the set $A$, we obtain
\begin{align*}
    \Pbb\left[\limsup_{T\to\infty} \frac{|A_i\cap\Xbb_{\leq T}|}{T} \geq \frac{\epsilon}{2},\forall i\geq 1 \right] &= \lim_{I\to\infty} \Pbb\left[\limsup_{T\to\infty} \frac{|A_i\cap\Xbb_{\leq T}|}{T} \geq \frac{\epsilon}{2},1\leq i\leq I \right]\\
    &= \lim_{I\to\infty}\Pbb\left[\limsup_{T\to\infty} \frac{|A_I\cap\Xbb_{\leq T}|}{T} \geq \frac{\epsilon}{2}\right]\\
    &\geq \frac{\epsilon}{2}.
\end{align*}
We will denote by $\Hcal$ this event in which for all $i\geq 1$, we have $\limsup_{T\to\infty} \frac{|A_i\cap\Xbb_{\leq T}|}{T}\geq \frac{\epsilon}{2}$. Under the event $\Hcal$, for any $i,t^0\geq 1$, there always exists $t^1>t^0$ such that $\frac{|A_i\cap\Xbb_{\leq t^1}|}{t^1} \geq \frac{\epsilon}{4}.$ We construct a sequence of times $(t_p)_{p\geq 1}$ and indices $(i_p)_{p\geq 1}$, $(u_p)_{p\geq 1}$ by induction as follows. We first pose $i_1=t_0=0$. Now assume that for $p\geq 1$, the time $t_{p-1}$ and index $i_p$ are defined. Let $t_p>t_{p-1}$ such that
\begin{equation*}
    \Pbb\left[\Hcal^c \cup \bigcup_{t_{p-1}<t\leq t_p}\left\{\frac{|A_{i_p}\cap\Xbb_{\leq t}|}{t}\geq \frac{\epsilon}{4}\right\}\right]\geq 1-\frac{\epsilon}{2^{p+3}}.
\end{equation*}
This is also possible because $\Hcal\subset \bigcup_{t>t_{p-1}}\left\{\frac{|A_{i_p}\cap\Xbb_{\leq t}|}{t} \geq \frac{\epsilon}{4}\right\}$. Last, let $i_{p+1}> i_p$ such that $\Pbb[A_{i_{p+1}}\cap \Xbb_{\leq t_p}\neq\emptyset]\leq \frac{\epsilon}{2^{p+3}}$ which is possible since $A_u\downarrow \emptyset$ as $u\to\infty$. We denote $\Ecal_p$ this event. Then,
\begin{align*}
    &\Pbb\left[\Hcal^c \cup \bigcup_{t_{p-1}<t\leq t_p}\left\{\frac{|(A_{i_p}\setminus A_{i_{p+1}})\cap\Xbb_{\leq t}|}{t} 
    \geq \frac{\epsilon}{4}\right\}\right]\\
    &\geq \Pbb\left[\Ecal_p\cap \Hcal^c \cup \bigcup_{t_{p-1}<t\leq t_p}\left\{\frac{|A_{i_p}\cap\Xbb_{\leq t}|}{t}\geq \frac{\epsilon}{4}\right\}\right] \geq 1-\frac{\epsilon}{2^{p+2}}.
\end{align*}
We denote $\Fcal_p$ this event. This ends the recursive construction of times $t_p$ and indices $i_p$ for all $p\geq 1$. Note that by construction, $\Pbb[\Fcal_p^c]\leq \frac{\epsilon}{2^{p+2}}$. Hence, by union bound, the event $\Hcal\cap\bigcap_{p\geq 1}\Fcal_p$ has probability $\Pbb[\Hcal\cap\bigcap_{p\geq 1}\Fcal_p]\geq \Pbb[\Hcal]-\frac{\epsilon}{4}\geq \frac{\epsilon}{4}$. For conciseness, denote $B_p = A_{i_p}\setminus A_{i_{p+1}}$. On the event $\Hcal\cap\bigcap_{p\geq 1}\Fcal_p$ we showed that for all $p\geq 1$, there exists $t_{p-1}<t\leq t_p$ such that $|B_p\cap\Xbb_{\leq t}| \geq \frac{\epsilon}{4} t$, and $(B_p)_{p\geq 1}$ is a sequence of disjoint measurable sets.

Now for any $p\geq 1$, we will construct a countable partition of $B_p$ that separates all points falling in $B_p$ within time horizon $t_p$. Let $\delta_p>0$ such that
\begin{equation*}
    \Pbb\left[\min_{u,v\leq t_p: X_u\neq X_v} \rho(X_u,X_v) \leq  \delta_p \right] \leq \frac{\epsilon}{2^{p+3}}.
\end{equation*}
We denote by $\Gcal_p$ the complementary of this event. Note that $\Pbb[\bigcup_{p\geq 1} \Gcal_p^c] \leq \frac{\epsilon}{8}$. As a result, the event $\Ical:=\Hcal\cap\bigcap_{p\geq 1}(\Fcal_p\cap\Gcal_p)$ has probability at least $\frac{\epsilon}{8}$. We will show that on this event, $\Xbb$ disproves the $\Ccal_2$ condition. Precisely, let $(x^i)_{i\geq 1}$ a dense sequence of $\Xcal$. We will denote the balls of $\Xcal$ by $B(x,r) = \{x':\rho(x,x')<r\}$. Define the following partition of $\Xcal$,
\begin{equation*}
    \Pcal(\delta):\quad P_i(\delta) = B(x^i,\delta) \setminus \bigcup_{j<i} B(x^j,\delta).
\end{equation*}
Finally, for any $p, i\geq 1$, define $P^p_i:= P_i(\delta_p)\cap B_p.$ We can note that $\bigcup_{i\geq 1} P^p_i = B_p$. Further, the sets $(B^p_i)_{i,p\geq 1}$ are all disjoint, and form a countable sequence. However, on the event $\Ical$, for every $p\geq 1$, there exists a time $t_{p-1}<t\leq t_p$ such that $|B_p\cap\Xbb_{\leq t}|\geq \frac{\epsilon}{4}t$. But because the event $\Gcal_p$ is satisfied, all the points falling in $B_p$ within horizon $t\leq t_p$ are separated by at least $\delta_p$, hence fall in distinct sets $B^p_i$. As a result,
\begin{equation*}
    |\{i\geq 1: P^p_i\cap \Xbb_{\leq t}\neq\emptyset\}| \geq |B_p\cap\Xbb_{\leq t}|\geq \frac{\epsilon}{4}t.
\end{equation*}
This shows that on the event $\Ical$, for every $p\geq 1$, there exists $t>t_{p-1}$ such that $|\{i,p\geq 1: P^p_i\cap \Xbb_{\leq t}\neq\emptyset\}| \geq \frac{\epsilon}{4}t$, and as a result 
\begin{equation*}
\limsup_{T\to\infty} \frac{|\{i,p\geq 1: P^p_i\cap \Xbb_{\leq T}\neq\emptyset\}|}{T} \geq \frac{\epsilon}{4}.
\end{equation*}
The fact that $\Pbb[\Ical]\geq \frac{\epsilon}{8}$ ends the proof that $\Xbb\notin\Ccal_2$, and that the first proposition is equivalent to $\Ccal_2$.

We now show the equivalence $(2)\Leftrightarrow(3)$. We clearly have $(3)\Rightarrow(2)$. Now suppose that $\Xbb$ satisfies $(2)$. Let $M>1$ and $A$ be a measurable set. Then, for any $T\geq 1$, we have
\begin{equation*}
    \frac{1}{T} \sum_{t\leq T,t\in\Tcal^{\leq M}} \1_A(X_t) \leq M\frac{|A\cap\Xbb_{\leq t}|}{T} = \frac{M}{T} \sum_{t\leq T,t\in\Tcal^{\leq 1}} \1_A(X_t).
\end{equation*}
Because $(X_t)_{t\in\Tcal^{\leq 1}}\in\Ccal_1$, we obtain as a result $(X_t)_{t\in\Tcal^{\leq M}}\in\Ccal_1$ using the definition. This ends the proof of the proposition.
\end{proof}

As a consequence of Proposition \ref{prop:C2_equivalent_forms}, we obtain new major insights on the noiseless full-feedback setting. In this setting, an online learning sequentially observes an instance $X_t\in\Xcal$, predicts a value $\hat Y_t\in\Ycal$ then observes the true value $Y_t=f^*(X_t)$ for some unknown measurable function $f:\Xcal\to\Ycal$. Similarly to the notion of universal consistence for contextual bandits, the goal is to find learning rules satisfying $\frac{1}{T}\sum_{t=1}^T \ell(Y_t,\hat Y_t) \to 0\quad(a.s.),$ where $\ell$ is a given near-metric on $\Ycal$. For this setting, \citep{hanneke:21} gave a algorithm combining the Hedge algorithm and a ``dense'' countable family of measurable functions, universally consistent under $\Ccal_1$ processes. \citep{blanchard:22a} then gave a simple 1-nearest-neighbor-based algorithm 2C1NN and showed that in general separable Borel metrizable spaces \cite{blanchard:22c}, it is universally consistent under $\Ccal_2$ processes, which are also necessary for universal learning \citep{hanneke:21}. Proposition \ref{prop:C2_equivalent_forms} directly implies that combining the original algorithm from \citep{hanneke:21} on new instances $X_t$, i.e., on times $\Tcal^{\leq 1}$, with memorization for previously observed instances also yields an optimistically universal learning rule. Unfortunately, such direct argument does not extend to a noisy setting \cite{blanchard:22d} where the values $Y_t$ may not come from a fixed measurable function $f^*(X_t)$.

\subsection{Learning with experts algorithms}

We give the main ingredients that will be used as sub-routine in our algorithms. We start by recalling classical result on the regret of $\EXP$.

\begin{theorem}[Expected regret of $\EXP$ \citep*{bubeck2012regret}] \label{thm:exp3}
    If $\EXP$ is run with parameters $\eta_t = \sqrt{\frac{\ln K}{tK}}$ on a multi-armed bandit with $K$ arms, then the pseudo regret satisfies
    \begin{equation*}
        \max_{i=1,\ldots,k}\Ebb\left[\sum_{t=1}^T r_i(t)\right] - \Ebb\left[\sum_{t=1}^T r_{\hat i_t}(t)\right]\leq 2\sqrt{TK\ln K}.
    \end{equation*}
\end{theorem}

We will also need an algorithm for adversarial multi-armed bandits that holds with high probability $1-\delta$, with parameters that do not depend on the confidence $\delta$ nor the horizon $T$.

\begin{theorem}[High-probability regret of $\EXPIX$ \cite{neu2015explore}] \label{thm:multiarmed_bandits}
There exists an algorithm $\EXPIX$ for adversarial multi-armed bandit with $K\geq 2$ arms such that for any $\delta\in(0,1)$ and $T\geq 1$,
\begin{equation*}
    \max_{i\in[K]} \sum_{t=1}^T( r_t(a_i) - r_t(\hat a_t)) \leq 4\sqrt{KT\ln K} + \left(2\sqrt{\frac{KT}{\ln K}}+1\right)\ln \frac{2}{\delta},
\end{equation*}
with probability at least $1-\delta$.
\end{theorem}
Specifically we will always use a very simplified version of this result. There exists a universal constant $c>0$ such that 
\begin{equation*}
    \max_{i\in[K]} \sum_{t=1}^T( r_t(a_i) - r_t(\hat a_t)) \leq c\sqrt{KT\ln K}\ln \frac{1}{\delta},
\end{equation*}
with probability $1-\delta$ for $\delta\leq \frac{1}{2}$. This has the following corollary which allows one to consider a countable family of experts asymptotically, based on an argument from \citep*[Corollary 4]{hanneke:22a}. We use the same construction to design an algorithm $\EXPINF$ for learning with a countably infinite number of experts---the original proof extended the Hedge algorithm to infinite number of experts in the full-feedback setting. Precisely, we use an increasing sequence of times $(T_i)_{i\geq 1}$ such that the learning rule performs an independent $\EXPIX$ algorithm during each period $[T_i,T_{i+1})$. During this period, the $\EXPIX$ learner is run with $i$ arms consisting in the experts $E_k$ for $k\leq i$. To ease the computations, we choose $T_i = \sum_{j< i} j^3 = \frac{i^2(i+1)^2}{4}$, which yields the following bounds.

\begin{corollary}\label{cor:infinite-exp4}
There is an online learning rule 
$\EXPINF$ using bandit feedback
such that for any countably infinite set of experts $\{E_1,E_2,\ldots\}$ 
(possibly randomized), for any $T\geq 1$ and $0<\delta\leq \frac{1}{2}$, with probability at least $1-\delta$,
\begin{equation*}
\max_{1\leq i \leq T^{1/8}} \sum_{t=1}^{T} \left( r_t(E_{i,t}) - r_t(\hat{a}_t) \right) 
\leq cT^{3/4}\sqrt{\ln T}\ln\frac{T}{\delta}.
\end{equation*}
where $c>0$ is a universal constant. Further, with probability one on the learning and the experts, there exists $\hat T$ such that for any $T\geq 1$,
\begin{equation*}
    \max_{1\leq i \leq T^{1/8}} \sum_{t=1}^{T} \left( r_t(E_{i,t}) - r_t(\hat{a}_t) \right) 
\leq \hat T +  cT^{3/4}\sqrt{\ln T}\ln T.
\end{equation*}
\end{corollary}

\begin{proof}
Denote by $(T_i =\sum_{j< i} j^3)_{i\geq 1}$ the restarting times used in the definition of $\EXPINF$, and by $\hat a_t$ its selected action at time $t$. Theorem \ref{thm:multiarmed_bandits} implies that for any $i\geq 1$, with probability at least $0<\delta<\frac{1}{2}$,
\begin{equation*}
    \max_{1\leq j\leq i} \sum_{t=T_i}^{T_{i+1}-1} r_t(E_{j,t})-r_t(\hat a_t) \leq c\sqrt{i(T_{i+1}-T_i)\ln i} \ln\frac{1}{\delta} = c i^2\sqrt{\ln i} \ln\frac{1}{\delta}.
\end{equation*}
Now fix $T\geq 1$ and $\delta>0$. Let $i\geq 0$ such that $T_{i+1}\leq T<T_{i+2}$. Then summing the above equations gives that with probability at least $\delta$,
\begin{align*}
    \max_{1\leq j\leq T^{1/8}}\sum_{t=1}^T r_t(E_{j,t})-r_t(\hat a_t) &\leq T_{\lceil T^{1/8}\rceil} + (T-T_{i+1})+ \sum_{t=T_{\lceil T^{1/8}\rceil}}^{T_{i+1}-1} r_t(E_{i,t})-r_t(\hat a_t) \\
    &\leq T_{\lceil T^{1/8}\rceil} + (i+1) + c\frac{i(i+1)(2i+1)}{6} \sqrt{\ln i}\ln\frac{i}{\delta}.
\end{align*}
Now note that $i\sim \sqrt 2 T^{1/4}$ and $T_{\lceil T^{1/8}\rceil}\sim \frac{\sqrt T}{4}$ as $T\to\infty$. Therefore, there exists a universal constant $\tilde c$ such that for all $T\geq 1$, the right-hand term is upper bounded by $\tilde c T^{3/4}\sqrt{\ln T}\ln\frac{T}{\delta}$. This ends the proof of the first claim.

Now for any $T\geq 1$, using the probabilities of error $\delta_T=\frac{1}{T^2}$ which are summable, the Borel-Cantelli lemma implies that on an event of probability one, there exists $\hat T$ such that for any $T\geq \hat T$,
\begin{equation*}
    \max_{1\leq j\leq T^{1/8}}\sum_{t=1}^T r_t(E_{j,t})-r_t(\hat a_t) \leq \tilde c T^{3/4}\sqrt{\ln T}\ln(T^3) = 3\tilde c T^{3/4}\sqrt{\ln T}\ln T,
\end{equation*}
which ends the proof of the second claim by redefining the constant $c>0$.
\end{proof}

\section{Finite action space}
\label{sec:finite-actions}

In this section, we assume that the action space $\Acal$ is finite and we show that in this case, the set of processes $\Xbb$ admitting universal learning is exactly $\Ccal_2$. In other terms, we can recover the same processes which admit universal learning in the full-feedback setting.

We start by showing that the $\Ccal_2$ condition is necessary for universal consistency, which is a direct consequence from its necessity in the full-feedback case \cite{hanneke:21}.

\begin{theorem}
\label{thm:SMV-finite-nec}
If $2 \leq |\A| < \infty$, 
$\ProcX \in \SMV$ is necessary 
for universal consistency, i.e., $\Ccal \subset \Ccal_2$.
\end{theorem}

\begin{proof}
In the full-information feedback setting, \citep*[Theorem 37]{hanneke:21} showed that $\Xbb\in\Ccal_2$ is necessary for universal learning even for noiseless responses in binary classification. We will present a simple reduction from the full-feedback to the partial-feedback setting. Precisely, let $a_0,a_1\in\Acal$ be two distinct actions. To any measurable function $f:\Xcal\to \{0,1\}$ we associate a deterministic reward function $r_f:\Xcal\times \Acal \to [0,1]$ as follows
\begin{equation*}
    r_f(x,a) = f(x)\1[a=a_1] + (1-f(x))\1[a=a_0],\quad x\in\Xcal,a\in\Acal.
\end{equation*}
Note that any action $a\in\Acal\setminus\{a_0,a_1\}$ always has reward $0$. Now suppose that for a process $\Xbb$ there exists an universally consistent learning rule $f_\cdot$ for contextual bandits. Then, we can consider the following learning rule for the complete-feedback setting, recursively defined as
\begin{equation*}
    \tilde f_t(\mb x_{\leq t-1},\mb y_{\leq t-1},x_t) = \1[f_t(\mb x_{\leq t-1},(\1[\tilde f_i(\mb x_{\leq i-1},\mb y_{\leq i-1},x_i) =y_i])_{i\leq t-1},x_t) = a_1].
\end{equation*}
for any $t\geq 1$, $\mb x_{\leq t}\in\Xcal^{t-1}$ and $\mb y_{\leq t-1}\in\{0,1\}^{t-1}$. We now shows that $\tilde f_\cdot$ is universally consistent for the noiseless full-feedback setting. For any measurable function $f:\Xcal\to\{0,1\}$, the learning rule $f_\cdot$ is consistent for the rewards $r_f$. In particular, if we denote by $\hat a_t$ the action selected by $f_\cdot$ at time $t$, using the measurable policy $\pi_f:x\in\Xcal\mapsto a_0\1[f(x)=0] + a_1\1[f(x)=1] \in\Acal$ which always selects the best action we obtain
\begin{equation*}
    \limsup_{T\to\infty}\frac{1}{T}\sum_{t=1}^T r_t(\pi_f(X_t))-r_t(\hat a_t) = \limsup_{T\to\infty} \frac{1}{T}\sum_{t=1}^T \1[\hat a_t \neq  \pi_f(X_t)] \leq 0,\quad (a.s.).
\end{equation*}
Now consider the actions $\hat a_t$ selected under $\Xbb$ and rewards $r_f$ and denote by $\tilde Y_t$ the prediction of $\tilde f_\cdot$ at time $t$ under $\Xbb$ and values $Y_t = f(X_t)$ for $t\geq 1$. By construction, for any $t\geq 1$, we have $\1[\hat a_t \neq \pi_f(X_t)] \geq  \1[\tilde Y_t \neq f(X_t)]$. Then, almost surely $\frac{1}{T}\sum_{t=1}^T \1[\tilde Y_t\neq f(X_t)] \underset{n\to\infty}{\longrightarrow} 0$. This shows that $\tilde f_\cdot$ is universally consistent for noiseless responses in binary classification, hence $\Xbb\in\Ccal_2$, which completes the proof. 
\end{proof}

We now present a learning rule for contextual bandits, which we will next show is universally consistent under any $\Ccal_2$ process.

This learning rule at time $t$ has different behaviour depending on the number of occurrences of $X_t$ that were observed in the past. Precisely, for any time $t$, we compute a corresponding category $p$ such that the number of past occurrences of $X_t$ belongs in the interval $[4^p,4^{p+1})$. The learning rule will behave completely separately on times from different categories. The formal definition is given by the function below
\begin{equation*}
    \textsc{Category}(t,\Xbb_{\leq t}) = \left\lfloor \log_4\left(\sum_{t'\leq t} \1[X_{t'}=X_t]\right)\right\rfloor.
\end{equation*}
For convenience we may write $\textsc{Category}(t)$ instead of $\textsc{Category}(t,\Xbb_{\leq t})$. Further, for a given category $p$, the algorithm will proceed by periods $[T_p^q,T_p^{q+1})$ defined as follows. For any $p\geq 0$ and $q\geq p2^p$, we define the times $T_p^q = 2^k + \frac{i}{2^p}2^k$, where $q = k 2^p + i$ with $0\leq i<2^p$. Note that the sequence $(T_p^q)_q$ has an exponential behaviour with rate between $2^{-p-1}$ and $2^{-p}$. We will refer to $[T_p^q,T_p^{q+1})$ as the period $q$ for category $p$. We then define the function $\textsc{Period}(t)$ which returns the index $q$ such that $T_p^q\leq t\leq T_p^{q+1}$ where $p$ is the category of $p$. Now let $(\pi^l)_{l\geq 1}$ be a sequence of measurable functions from $\Xcal$ to $\Acal$ that are dense within measurable functions under $\Ccal_1$ processes. Intuitively, the learning rule combines two strategies: a strategy 0 which applies a separate $\EXP$ algorithm to each distinct instance, and a strategy 1 which performs the best policy within a subset of the policies $(\pi^l)_{l\geq 1}$. In order to know which strategy to apply, the learning rule computes an estimate of the counterfactual loss of strategy $i$, using classical importance sampling on some allocated exploration times for strategy $i$. On the exploitation times, the learning rule uses these estimates to perform the best strategy. 

We first define the procedure \textsc{AssignPurpose} which taking as input a time $t$ determines whether this time will be used for exploration of strategy $0$ (output 0), strategy $1$ (output 1), or exploitation (output 2). Intuitively, \textsc{AssignPurpose} selects exploration times randomly with small probability while ensuring that times $t,t'$ from the same category $p$, period $q$, and that are duplicates $X_t=X_{t'}$ are assigned the same output, hence will serve for the same exploration or exploitation purpose. The algorithm is formally defined in Algorithm \ref{alg:AssignPurpose}. 

\begin{algorithm}
\caption{\textsc{AssignPurpose}}\label{alg:AssignPurpose}
\hrule height\algoheightrule\kern3pt\relax

\textbf{Input:} time $t$, $\Xbb_{\leq t}$, $\textsc{Category}(t')$ for $t'\leq t$, $\textsc{AssignPurpose}(t')$ for $t'<t$.\\
\textbf{Output:} $\textsc{AssignPurpose}(t)\in\{0,1,2\}$.\\

$p= \textsc{Category}(t)$; $q=\textsc{Period}(t)$\\
\uIf(\tcp*[f]{Not the first occurrence of $X_t$ in current period}){exists $t'<t$ with $ \textsc{Category}(t')=p$; $\textsc{Period}(t')=q$ and $X_t=X_{t'}$}{
    Return $\textsc{AssignPurpose}(t')$
}
\Else(\tcp*[f]{First occurrence of $X_t$ in current period}){
    $p_t = 1/(2t^{1/4})$\\
    $U_t \sim\Ucal([0,1])$\\
    \textbf{if} $U_t\leq p_t$ \textbf{then} Return 0 \tcp*[f]{Exploration for strategy 0}\\
    \textbf{else if} $p_t<U_t\leq 2p_t$ \textbf{then} Return 1 \tcp*[f]{Exploration for strategy 1}\\
    \textbf{else} Return 2 \tcp*[f]{Exploitation}
}
\hrule height\algoheightrule\kern3pt\relax
\end{algorithm}

Next, we define the subroutine $\textsc{Explore}(i;t)$ that will be called on exploration times $t$ for strategy $i$. We first define it to estimate the performance of strategy 0. The subroutine updates an estimator $\hat R_p^0(q)$ of the loss that would be incurred by using strategy 0 for all times in category $p$ during period $q$. $\textsc{Explore}(0,\cdot)$ is defined formally in Algorithm \ref{alg:Explore0}.

\begin{algorithm}
\caption{$\textsc{Explore}(0;\cdot)$}\label{alg:Explore0}
\hrule height\algoheightrule\kern3pt\relax
\textbf{Input:} time $t$, $\Xbb_{\leq t}$, $\textsc{Category}(t')$ for $t'\leq t$, rewards $\mb r_{<t}$, $\hat R_p^0(q)$ for $p\geq 0,q\geq p2^p$.\\
\textbf{Output:} Selects action $\hat a_t$ and updates $\hat R_p^0(q)$ for $p= \textsc{Category}(t)$, $q=\textsc{Period}(t)$.\\

$p= \textsc{Category}(t)$, $q=\textsc{Period}(t)$\\
$S_t = \{t'<t:\textsc{Category}(t') = p, \textsc{Period}(t')=q, X_{t'}=X_t\}$\\
$\hat a_t = \EXP_{\Acal}(\mb{\hat a}_{S_t}, \mb r_{S_t})$\\
Receive reward $r_t$\\
 Let  $t'=\min S_t$ \tcp*[f]{First occurrence of $X_t$}\\
$\hat R_p^0(q)\gets \hat R_p^0(q) + \frac{r_t}{p_{t'}}$ \tcp*[f]{Update estimate $\hat R_p^0(q)$}

\hrule height\algoheightrule\kern3pt\relax
\end{algorithm}

Then, we define $\textsc{Explore}(1,\cdot)$. It updates an estimator $\hat R^l_p(q)$ of the loss that would have been incurred using the policy $\pi^l$ for all times in category $p$ during period $q$, for all $l\geq 1$. Because there is an infinite number of such policies, they are introduced sequentially in the estimation process. \textsc{Explore} is defined formally in Algorithm \ref{alg:Explore1}.

\begin{algorithm}
\caption{$\textsc{Explore}(1;\cdot)$}\label{alg:Explore1}
\hrule height\algoheightrule\kern3pt\relax
\textbf{Input:} time $t$, $\Xbb_{\leq t}$, $\textsc{Category}(t')$ for $t'\leq t$, rewards $\mb r_{<t}$, $\hat R_p^l(q)$ for $l\geq 1,p\geq 0,q\geq p2^p$.\\
\textbf{Output:} Selects action $\hat a_t$ and updates $\hat R_p^l(q)$ for $p= \textsc{Category}(t)$, $q=\textsc{Period}(t)$.\\

$p=\textsc{Category}(t)$, $q=\textsc{Period}(t)$, $k = \lfloor \log_2 t \rfloor$\\
$l_t = \Ucal(\{1,\ldots,k\})$ \tcp*[f]{Uniform exploration} \\
$\hat a_t = \pi^{l_t}(X_t)$\\
Receive reward $r_t$\\
Let $t'=\min \{s<t:\textsc{Category}(s) = p, \textsc{Period}(s)=q, X_s=X_t\}$ \tcp*[f]{First occurrence of $X_t$}\\
$\hat R_p^l(q) \gets \hat R_p^l(q) + \frac{k}{p_{t'}}r_t\1[l=l_t],\quad 1\leq l\leq k$ \tcp*[f]{Update estimate $\hat R_p^{l_t}(q)$}

\hrule height\algoheightrule\kern3pt\relax
\end{algorithm}

The estimates $\hat R^l_p(q)$ updated by \textsc{Explore} are then used to select the strategy to perform on exploitation times. The learning rule that we will define acts separately on times from different categories: for any category $p\geq 0$, before starting phase $q$, the learning rule commits to performing strategy $\Pcal_p(q)\in\{0,1\}$, for times of that phase $q$ for category $p$. The choice of strategy $\Pcal_p(q)$ is performed by a subroutine \textsc{SelectStrategy} which applies an $\eta_p=\Ocal(2^{-p/2})$ average reward penalty for strategy $0$ then select the strategy that obtained the highest adjusted estimated reward during the previous period. Last, if during the current period $q$, strategy $0$ obtained the highest adjusted reward, we select this strategy for the future periods $q<q'\leq q+p2^p$. This ensures that if by mistake the rule selected $\Pcal_p(q)=1$, the loss incurred during this period is mitigated for the next strategy selection: the current performance until time $T^{q+1}$ is negligible up to a small average loss starting from time $T_p^{q+2^p+1}$. The construction of \textsc{SelectStrategy} is detailed in Algorithm \ref{alg:SelectStrategy}.

\begin{algorithm}
\caption{\textsc{SelectStrategy}}\label{alg:SelectStrategy}
\hrule height\algoheightrule\kern3pt\relax
\textbf{Input:} Category $p$, phase $q$, variable states $\hat R^l_p(t)$ for $t<T^{q+1}_p$\\
\textbf{Output:} Selects strategy $\Pcal_p(r)$ for some future phases $r>q$.\\

$\eta_p = 10\frac{\sqrt{|\Acal|\ln|\Acal|}}{2^{p/4}}$, $k = \lfloor\log_2 T_p^q \rfloor$\\
\If{$\Pcal_p(q+1)$ has not been defined yet}{
    \eIf{$\displaystyle \hat R_p^0(q) - \eta_p (T_p^{q+1}-T_p^q) \geq \max_{1\leq l\leq k} \hat R_p^l(q)$}{
        $\Pcal_p(q')=0,\quad q<q'\leq q+p2^p$ \tcp*[f]{perform strategy 0 until current performance is negligible up to a $\Ocal(2^{-p})$ average loss}
    }{
        $\Pcal_p(q+1)=1$
    }
    
}

\hrule height\algoheightrule\kern3pt\relax
\end{algorithm}

We are now ready to define the learning rule for stochastic rewards. On exploration times, the learning rule calls the subroutine \textsc{Explore}, and on exploitation times, the learning rule performs the corresponding strategy $\Pcal_p(q)$ for times in category $p$ during phase $q$. The construction of the learning rule is detailed in Algorithm \ref{alg:main_learning_rule}.

\begin{algorithm}
\caption{An optimistically universal learning rule for stochastic rewards}\label{alg:main_learning_rule}
\hrule height\algoheightrule\kern3pt\relax
$\hat R^l_p=0,  l\geq 0,p\geq 0$; $\Pcal_p(p2^{p+5})=0, p\geq 0$ \tcp*[f]{Initialization}\\

\For{$t\geq 1$}{
    Observe context $X_t$\\
    $p = \textsc{Category}(t)$, $q=\textsc{Period}(t)$\\
    \uIf(\tcp*[f]{Initially perform strategy 0 without period restriction}){$t<2^{32p}$}{
        $S_t = \{t'<t:\textsc{Category}(t')=p, X_{t'}=X_t\}$\\
        $\hat a_t = \EXP_{\Acal}(\mb{\hat a}_{S_t}, \mb r_{S_t})$
    }
    \uElseIf{$i:=\textsc{AssignPurpose}(t)\leq 1$}{
        \textsc{$\textsc{Explore}(i;t)$}
    }
    \Else(\tcp*[f]{Perform strategy $\Pcal_p(q)$}){
        \eIf{$\Pcal_p(q) = 0$}{
            $S_t = \{t'<t:\textsc{Category}(t')=p, \textsc{Period}(t')=q, X_{t'}=X_t\}$\\
            $\hat a_t = \EXP_{\Acal}(\mb{\hat a}_{S_t}, \mb r_{S_t})$
        }{
            $k = \lfloor\log_2 T_p^q \rfloor$ \\
            $S_t = \{t'<t:\textsc{Category}(t')=p, \textsc{Period}(t')=q,\textsc{AssignPurpose}(t')=2\}$\\
            $l_t = \EXPIX_{\{1,\ldots,k\}}\left(\mb l_{S_t}, \mb r_{S_t}\right)$ \tcp*[f]{Select policy $\pi^{l_t}$} \\
            $\hat a_t = \pi^{l_t}(X_t)$
        }
        Receive reward $r_t$
    }
    $\Ecal=\{(p',q'): q'\geq p'2^{p'+5}, t=T^{q'+1}_{p'}-1\}$\\
    \For{$(p',q') \in \Ecal$}{
        $\textsc{SelectStrategy}(p',q')$ \tcp*[f]{At the end of a phase $[T^{q'}_{p'},T^{q'-1}_{p'})$, select strategy for future phases}
    }
}

\hrule height\algoheightrule\kern3pt\relax
\end{algorithm}

The main result of this section is that this learning rule is optimistically universal.
\begin{theorem}\label{thm:opt_rule_stat}
    Let $\Xcal$ a metrizable separable Borel space and $\Acal$ a finite action set. Then, there exists an optimistically universal learning rule and the set of learnable processes is $\Ccal =\Ccal_2$.
\end{theorem}

\begin{proof}
We will denote by $\hat a_t$ the action selected by the learning rule at time $t$. For any $p\geq 0$, we define the set  $\Tcal_p$ of times in category $p$ as follows
\begin{equation*}
    \Tcal_p = \left\{t\geq 1: 4^p\leq \sum_{t'\leq t} \1[X_{t'} = X_t] <4^{p+1}\right\},
\end{equation*}
i.e. the set of times which correspond to duplicates of index in $[4^p,4^{p+1})$. We also define
\begin{align*}
    \Tcal_p^{exp,i} &= \{t\geq 2^{32p}: \textsc{AssignPurpose}(t)=i\},\quad i\in\{0,1\},\\
    \tilde\Tcal_p &=\{t\geq 2^{32p}: \textsc{AssignPurpose}(t)=2\},
\end{align*}
the set of exploration times for strategy $i$ in category $p$, and exploitation times in category $p$, respectively. For convenience, we also define $\Tcal_p(q) = \Tcal_p\cap[T_p^q,T_p^{q+1})$ times in category $p$ and phase $q$. Last, we define $A_p(q)=|\Tcal_p(q)\cap(\Tcal_p^{exp,0}\cup\Tcal_p^{exp,1})|$ the number of exploration times in period $q$ for category $p$.

Now fix a process $\Xbb\in\Ccal_2$ and let $r$ be a reward mechanism on $\Acal\times\Xcal$. We recall the notation $\bar r(\cdot,\cdot) = \Ebb[r(\cdot,\cdot)]$ for the average reward. We aim to show that $f_\cdot$ is consistent under $\Xbb$ for the rewards given by $r$. We first define the policy $\pi^*$ given by
\begin{equation*}
    \pi^*(x) = \argmax_{a\in \Acal}\bar r(a, x),
\end{equation*}
where ties are broken by the lexicographic rule. This function is measurable given that $\Acal$ is finite. Further, it is an optimal policy in the sense that for any measurable function $\pi:\Xcal\to\Acal$ and any $x\in\Xcal$, $\bar r(\pi(x), x)\leq \bar r(\pi^*(x), x)$.

For $p\geq 0$, we first analyze the reward estimates $\hat R_p^l(q)$ for $q\geq p2^{p+5}$ ($T_p^{p2^{p+5}}=2^{32p}$) and $l\geq 0$. First note that the exploration times $\Tcal_p^{exp,0}$ and $\Tcal_p^{exp,1}$ were constructed precisely so that times corresponding to the same instance and within the same period, fall in the same set $\Tcal_p^{exp,0}$, $\Tcal_p^{exp,1}$, or $\tilde \Tcal_p$. For simplicity, we will write $\Xcal_p(q) = \{X_t,t\in\Tcal_p(q)\}$ the set of visited instances during period $q$ of category $p$, and for $x\in\Xcal_p(q)$ we denote $t_p(q;x) = \min\{t\in\Tcal_p(q):X_t=x\}$ the first time of occurrence of $x$ in period $q$. Then, we can write
\begin{equation*}
    \hat R_p^0(q) = \sum_{x\in \Xcal_p(q)} \frac{\1[U_{t_p(q;x)}\leq p_{t_p(q;x)}]}{p_{t_p(q;x)}} \sum_{t\in\Tcal_p(q), X_t=x} \tilde r_t
\end{equation*}
where $\tilde r_t$ is the reward at time $t$ that would have been obtained by performing strategy 0 during period $q$, i.e., assigning an independent $\EXP$ learner for each different instance in this period. We compare $\hat R_p^0(T)$ to the average reward obtained by the optimal policy $\pi^*$,
\begin{equation*}
    \bar R_p^*(q):=\sum_{t\in\Tcal_p(q)} \bar r(\pi^*(X_t), X_t).
\end{equation*}
Observe that conditionally on $\Xbb$, the terms in the sum of $\hat R_p^0(q)$ are independent. For any $x\in\Xcal_p(q)$, let $\bar R_p^0(q;x) = \Ebb[\sum_{t\in\Tcal_p(q), X_t=x} \tilde r_t\mid\Xbb]$, the average reward obtained by strategy $0$ on the instance $x$. We will use the notation $N_p(q;x) =|\{t\in\Tcal_p(q),X_t=x\}|\leq 4^{p+1} $ for the number of occurrences of the instance $x$ within $\Tcal_p$. Note that
\begin{equation*}
    \left|\frac{\1[U_{t_p(q;x)}\leq p_{t_p(q;x)}]}{p_{t_p(q;x)}} \sum_{t\in\Tcal_p(q), X_t=x} \tilde r_t\right|\leq \frac{N_p(q;x)}{p_{t_p(q;x)}}\leq 2^{2p+3} (T_p^{q+1})^{1/4},
\end{equation*}
and that $|\Xcal_p(q)|\leq \frac{T_p^{q+1}}{2^{2p}}$ since by definition of $\Tcal_p$ each instance has already occurred $4^p$ times. As a result, we can apply Hoeffding's inequality to obtain
\begin{equation*}
    \Pbb\left[\left|\hat R_p^0(q)- \sum_{x\in\Xcal_p(q)}  \bar R_p^0(q;x)\right|\leq (T_p^{q+1})^{\frac{7}{8}}  \mid\Xbb\right] \geq 1-2\exp\left(-\frac{(T_p^{q+1})^{1/4}}{2^{2p+5}}\right) := 1-2p_1(p,q)
\end{equation*}
Now applying Theorem \ref{thm:exp3} to each pseudo-regret $\bar R_p^0(q;x)$ yields
\begin{align*}
    \sum_{x\in\Xcal_p(q)}& \bar R_p^0(q;x) \\
    &\geq \sum_{x\in\Xcal_p(q)} N_p(q;x)\left( \max_{a\in\Acal} \bar r(a, x) - 2\sqrt{\frac{ |\Acal|\ln |\Acal|}{N_p(q;x)}} \right)\\
    &\geq \bar R_p^*(q) - 2\sqrt{\frac{ |\Acal|\ln |\Acal|}{2^{p/2}}}(T_p^{q+1}-T_p^q) - 2\sqrt{|\Acal|\ln |\Acal|} \sum_{x\in\Xcal_p(q),N_p(q;x)\leq 2^{p/2}} N_p(q;x)\\
    &\geq \bar R_p^*(q) - 2\frac{\sqrt{ |\Acal|\ln |\Acal|}}{2^{p/4}}(T_p^{q+1}-T_p^q) - 2\sqrt{|\Acal|\ln |\Acal|} \frac{2^{p/2}}{4^p} T_p^{q+1}\\
    &\geq \bar R_p^*(q) - 6\frac{\sqrt{ |\Acal|\ln |\Acal|}}{2^{p/4}}(T_p^{q+1}-T_p^q).
\end{align*}
where in the third inequality, we used the fact that instances appearing in $\Tcal_p$ before $T_p^{q+1}$ are visited at least $4^p$ times before horizon $T_p^{q+1}$, by construction of $\Tcal_p$; and in the last inequality we used $2^{-p-1}T_p^{q+1}\leq T_p^{q+1}-T_p^q\leq 2^{-p}T_p^q$. Also, note that $\bar R_p^*(q)\geq \sum_{x\in\Xcal_p(q)} \bar R_p^0(q;x)$. As a result, taking the expectation over $\Xbb$, we obtain that with probability at least $ 1-2p_1(p,q)$,
\begin{equation}\label{eq:exp0_estimate}
    \left|\hat R_p^0(q) -  \bar R_p^*(q)\right|\leq  (T_p^{q+1})^{\frac{7}{8}} + 6\frac{\sqrt{ |\Acal|\ln |\Acal|}}{2^{p/4}}(T_p^{q+1}-T_p^q).
\end{equation}
Now consider the quantity $\tilde R^0_p(q)$, the reward that would be obtained for exploitation times on period $q$ if strategy 0 was applied. We have
\begin{align*}
    \tilde R^0_p(q) &= \sum_{x\in\Xcal_p(q)} \sum_{t\in\Tcal_p(q), X_t=x, t\in\tilde \Tcal_p}\tilde r_t\\
    &\geq \sum_{x\in\Xcal_p(q)} \sum_{t\in\Tcal_p(q), X_t=x}\tilde r_t - A_p(q).
\end{align*}
Similarly as above, using Hoeffding's inequality, we have
\begin{equation*}
    \Pbb\left[\sum_{x\in\Xcal_p(q)} \sum_{t\in\Tcal_p(q), X_t=x}\tilde r_t \geq  \sum_{x\in\Xcal_p(q)}\bar R^0_p(q;x) - (T_p^{q+1})^{3/4}\right] \geq 1-e^{-\frac{\sqrt {T_p^{q+1}}}{2^{2p+3}}} :=1-p_2(p,q).
\end{equation*}
As a result, with probability $1-p_2(p,q)$, we have
\begin{equation}\label{eq:estimate_strategy0_exploitation}
    \tilde R^0_p(q) \geq \bar R_p^*(q) - 6\frac{\sqrt{ |\Acal|\ln |\Acal|}}{2^{p/4}}(T_p^{q+1}-T_p^q) - (T_p^{q+1})^{3/4} - A_p(q).
\end{equation}

We now turn to the estimates $\hat R_p^l(q)$ for $l\geq 1$. Note that the estimation of $R_p^l(q)$ only starts at time $2^l$. Hence, we can consider $k(q)=\lfloor\log_2 T_p^q\rfloor = \lfloor \frac{q}{2^p}\rfloor$ and observe that during period $q$, the only estimates $\hat R_p^l(q)$ that are considered are for $1\leq l\leq k(q)$. Therefore, similarly as for the estimates $\hat R_p^0(q)$, we can write for $q\geq p2^{p+5}$ and $1\leq l\leq k(q)$,
\begin{equation*}
    \hat R_p^l(q) = \sum_{x\in\Xcal_p(q)} \frac{\1[p_{t_p(q;x)}< U_{t_p(q;x)}\leq 2p_{t_p(q;x)}]}{p_{t_p(q;x)}} \sum_{t\in\Tcal_p(q) X_t=x}  k(t)\1[l=l_t] r(\pi^l(x), x),
\end{equation*}
where $k(t)$ is the number of policies $\pi^l$ tested at time $t$, i.e. $k(t) = \lfloor \log_2 t\rfloor$. Conditionally on $\Xbb$ and $\mb U$ we can apply Hoeffding's inequality to obtain
\begin{align*}
    &\Pbb\left[\left|\hat R_p^l(q) - \sum_{x\in\Xcal_p(q)} \sum_{t\in\Tcal_p(q), X_t=x} \frac{\1[p_{t_p(q;x)}< U_{t_p(q;x)}\leq 2p_{t_p(q;x)}]}{p_{t_p(q;x)}}  \bar r(\pi^l(x), x) \right| \right.\\
    &\quad\quad\quad\quad\quad\quad\quad\quad\quad\quad\quad\quad\quad\quad\quad\quad\quad\quad\quad\quad\quad\quad\quad\quad\quad\quad\quad\quad \left.\leq (T_p^{q+1})^{7/8}\mid \Xbb,\mb U\right]\\
    &\geq 1- 2e^{-\frac{2(T_p^{q+1})^{7/4}}{(T_p^{q+1}-T_p^q)4(\log_2 T_p^{q+1})^2 \sqrt {T_p^{q+1}}}} \geq 1-2e^{-\frac{2^p(T_p^{q+1})^{1/4}}{4(\log_2 T_p^{q+1})^2}}:=1-2p_3(p,q).
\end{align*}
For convenience, let us denote by $\hat R^l_{p,bis}(q)$ the sum in the above inequality. We also define $\bar R^l_p(q) =\sum_{t\in\Tcal_p(q)} \bar r(\pi^l(X_t), X_t)$ the expected reward of policy $l$ on period $q$. Now, similarly as before, we have
\begin{equation*}
    0\leq \sum_{t\in\Tcal_p(q), X_t=x} \frac{\1[p_{t_p(q;x)}< U_{t_p(q;x)}\leq 2p_{t_p(q;x)}]}{p_{t_p(q;x)}}  \bar r(\pi^l(x), x) \leq \frac{N_p(q;x)}{p_{t_p(q;x)}}\leq 2^{2p+3} (T_p^{q+1})^{1/4}.
\end{equation*}
As a result, conditionally on $\Xbb$, Hoeffding's inequality yields
\begin{equation*}
    \Pbb[| \hat R^l_{p,bis}(q) - \bar R^l_p(q) |\leq (T_p^{q+1})^{7/8}\mid \Xbb]    \geq 1- 2p_1(p,q).
\end{equation*}
Thus, with probability at least $1-2p_1(p,q)-2p_3(p,q)$ we have
\begin{equation}\label{eq:exp1_estimates}
    |\hat R_p^l(q)  - \bar R^l_p(q)|\leq (T_p^{q+1})^{7/8}.
\end{equation}
Next, we consider the quantity $\tilde R^1_p(q)$, the reward that would have been obtained for exploitation times on period $q$ if strategy 1 was applied. Then, using Theorem \ref{thm:multiarmed_bandits}, we have with probability at least $1-e^{-(T_p^{q+1})^{1/4}}:=1-p_4(p,q)$,
\begin{align*}
    \max_{1\leq l\leq k(q)} \sum_{t\in\Tcal_p(q)\cap\tilde\Tcal_p}r_t(\pi^l(X_t), X_t) -\tilde R^1_p(q) &\leq c\sqrt{k(q)\ln k(q)(T_p^{q+1}-T_p^q)}(T_p^{q+1})^{1/4}\\
    &\leq c(T_p^{q+1})^{3/4} \ln T_p^{q+1}.
\end{align*}
As a result, we have
\begin{equation*}
    \tilde R^1_p(q) \geq \max_{1\leq l\leq k(q)} \sum_{t\in\Tcal_p(q)}r_t(\pi^l(X_t), X_t) - c(T_p^{q+1})^{3/4} \ln T_p^{q+1} - A_p(q).
\end{equation*}
Now, by Hoeffding's inequality, for every $1\leq l\leq k(q)$, with probability at least $1-e^{-2^p \sqrt{T_p^{q+1}}}:=1-p_5(p,q)$,
\begin{equation*}
    \sum_{t\in\Tcal_p(q)}r_t(\pi^l(X_t), X_t) \geq \bar R^l_p(q) - (T_p^{q+1})^{3/4}.
\end{equation*}
Hence, with probability $1-p_4(p,q)-k(q)p_5(p,q)$ we have
\begin{equation}\label{eq:estimate_strategy1_exploitation}
    \tilde R^1_p(q) \geq \max_{1\leq l\leq k(q)}\bar R^l_p(q) -  (T_p^{q+1})^{3/4} - c(T_p^{q+1})^{3/4} \ln T_p^{q+1} - A_p(q).
\end{equation}
We will also need the quantity $\tilde R^1_p(q;T)$ for $T_p^q\leq T<T_p^{q+1}$ which is the reward that would have been obtained for exploitation times from $T_p^q$ to $T$. The exact same arguments as above show that with probability at least $1-p_4(p,q)-k(q)p_5(p,q)$ we have
\begin{multline}\label{eq:exploitation_1_tail}
    \tilde R^1_p(q;T) \geq \max_{1\leq l\leq k(q)} \sum_{t\in\Tcal_p(q),t\leq T}\bar r(\pi^l(X_t), X_t) -  (T_p^{q+1})^{3/4} - c(T_p^{q+1})^{3/4} \ln T_p^{q+1} \\
    - A_p(q).
\end{multline}
Last, we now bound the exploration terms $A_p(q)$ to show that exploration times are negligible. Writing $A_p(q) = \sum_{x\in\Xcal_p(q)}\1[U_{t_p(q;x)}\leq 2p_{t_p(q;x)}]N_p(q;x)$, and because $\frac{N_p(q;x)}{t_p(q;x)^{1/4}}\leq 2^{2p+2}(T_p^{q+1})^{1/4}$, using Hoeffding's inequality we obtain that with probability at least $1-e^{-\frac{(T_p^{q+1})^{1/4}}{2^{2p+3}}}:=1-p_6(p,q)$,
\begin{equation}\label{eq:exploration_bound}
    A_p(q) \leq \sum_{x\in\Xcal_p(q)} \frac{N_p(q;x)}{t_p(q;x)^{1/4}} + (T_p^{q+1})^{7/8}
    \leq \frac{T_p^{q+1}-T_p^q}{(T_p^q)^{1/4}} + (T_p^{q+1})^{7/8}
    \leq 2(T_p^{q+1})^{7/8}.
\end{equation}
Now recalling that $k(q)\leq \frac{q}{2^p}$, we have that
\begin{multline*}
    \sum_{p\geq 0} \sum_{q\geq p2^{p+5}} 2p_1(p,q)+p_2(p,q) +p_6(p,q)+ k(q) (2p_1(p,q)+2p_3(p,q))  \\
    +(p_4(p,q)+k(q)p_5(p,q))(1+T_p^{q+1}-T_p^q)<\infty.
\end{multline*}
As a result, the Borel-Cantelli lemma implies that on an event $\Ecal$ of probability one, there exists $\hat T_1$ such that for any $p\geq 0$, $q\geq p2^{p+5}$ Eq~\eqref{eq:exp0_estimate}, \eqref{eq:estimate_strategy0_exploitation}, \eqref{eq:estimate_strategy1_exploitation} and \eqref{eq:exploration_bound} are satisfied, and \eqref{eq:exp1_estimates} is satisfied for $q\geq l,p2^{p+5}$, and Eq~\eqref{eq:exploitation_1_tail} is satisfied for $T_p^q\leq T<T_p^{q+1}$.

We are now ready to prove the universal consistence of the learning rule. First, we pose $\epsilon_p = 2\frac{\sqrt{|\Acal|\ln|\Acal|}}{2^{p/4}}$ and aim to show that the average error made by the learning rule on $\Tcal_p$ is $\Ocal(\epsilon_p)$ uniformly over time. Note in particular that $\sum_{p\geq 0}\epsilon_p<\infty$. For any $T\geq 1$, we define $\Rcal_p(T) = \sum_{t\leq T,t\in\Tcal_p}r_t$ the reward obtained by the learning rule, and $\bar R^*_p(T)=\sum_{t\leq T,t\in\Tcal_p}\bar r(\pi^*(X_t), X_t)$ the reward obtained by the optimal policy. To do so, we first start by analyzing the regret on the first period $[1,2^{32p})$ where there is no exploration and the learning rule uses $\EXPIX$ learners on each new instance. For $T<2^{32p}$ let $\Xcal_p(T):=\{X_t,t\in\Tcal_p, t\leq T\}$. Note that $|\Xcal_p(T)|\leq\frac{T}{4^p}$ by definition of $\Tcal_p$. For $x\in\Xcal_p(T)$, let $N_p(T;x)=|\{t\leq T, t\in\Tcal_p, X_t=x\}|\leq 2^{2p+2}$ and $\bar R^0_p(T;x):=\Ebb[\sum_{t\leq T,t\in\Tcal_p, X_t=x} \tilde r_t \mid \Xbb]$ where $\tilde r_t$ is the reward obtained if we used strategy 0. Now by Theorem \ref{thm:multiarmed_bandits}, for every $x\in\Xcal_p(T)$, with probability at least $1-e^{- p^2T^{1/2^7}}$, we have
\begin{equation*}
    \sum_{t\leq T, t\in\Tcal_p(T),X_t=x} r_t(\pi^*(x), x) - \Rcal_p(T)  \leq cp T^{1/2^7}\sqrt{|\Acal|\ln|\Acal|N_p(T;x)}
\end{equation*}
As a result, with probability at least $1-Te^{- p^2T^{1/2^7}}:=1-p_7(p,T)$,
\begin{align*}
    \sum_{t\leq T,t\in\Tcal_p}r_t (\pi^*(x), x)-r_t  
    &\leq \frac{2^p}{4^p}T + \sum_{x\in\Xcal_p(T),N_p(T;x)\geq 2^p} \sum_{t\leq T,t\in\Tcal_p,X_t=x}r_t (\pi^*(x), x)-r_t \\
    &\leq \frac{T}{2^p} +  cp\sqrt{|\Acal|\ln|\Acal|}T^{1/2^7} \sum_{x\in\Xcal_p(T), N_p(T;x)\geq 2^p}\sqrt{N_p(T;x)}\\
    &\leq \frac{T}{2^p} +  cp\sqrt{|\Acal|\ln|\Acal|}T^{1/2^7} \frac{T}{2^{p/2}}\\
    &\leq \frac{T}{2^p} +  \frac{c}{2}\sqrt{|\Acal|\ln|\Acal|}T^{1-1/2^7}\log_2 T,
\end{align*}
where in the last inequality, we used $2^{2p}\leq T<2^{32p}$, thus $2^{p/2}\geq T^{1/64}$. Then, by Hoeffding's inequality, we have with probability $1-e^{-2p^2\sqrt T}:=1-p_8(p,T)$,
\begin{equation*}
    \sum_{t\leq T,t\in\Tcal_p}r_t (\pi^*(x), x)\geq \bar R^*_p(T) - \frac{\log_2 T}{2}T^{3/4}.
\end{equation*}
Finally, with probability at least $1-p_7(p,T)-p_8(p,T)$, we obtain
\begin{equation}\label{eq:performance_initial_phase}
    \Rcal_p(T) \geq \bar R^*_p(T) -  \frac{1+c}{2}\sqrt{|\Acal|\ln|\Acal|}T^{1-1/2^7}\log_2 T -\frac{T}{2^p}.
\end{equation}
Noting that $\sum_{p\geq 0}\sum_{T\geq 1}p_7(p,T)+p_8(p,T)<\infty$, the Borel-Cantelli lemma implies that on an event $\Fcal$ of probability one, there exists $\hat T_2$ such that for all $T\geq \hat T_2$, and $p\geq 0$ such that $T<2^{32p}$, Eq~\eqref{eq:performance_initial_phase} holds. We will now suppose that the event $\Ecal\cap\Fcal$ of probability one is met.

Next we consider the case of $T\geq 2^{32p}$, and let $q_0\geq p2^{p+5}$ such that $T^{q_0}_p\leq T<T^{q_0+1}_p$. Then, consider
\begin{equation*}
    \Scal^0_p := \left\{p2^{p+5}\leq q<q_0: \hat R_p^0(q) - \eta_p (T_p^{q+1} - T_p^q)
    \geq \max_{1\leq l\leq k(q)} \hat R_p^k(q)\right\},
\end{equation*}
the set of phases where the learning rule estimated that strategy $0$ performed better than strategy $1$. Next, let $\Pcal^i_p = \{p2^{p+5}\leq q<q_0: \Pcal_p(q)=i\}$ the set of phases where the learning rule performed strategy $i$ for $i\in\{0,1\}$. An important observation is that for two phases $q_1<q_2\in\Scal^0_p\cap\Pcal^1_p$, if strategy $1$ should not have been performed, then $q_2>q_1+p2^p$. In particular, we have $T^{q_1}_p\leq 2^{-p} T^{q_2}_p$, hence $T^{q_1+1}_p-T^{q_1}_p \leq  2^{-p}(T^{q_2+1}_p-T^{q_2}_p)$. This allows to dissipate the errors made during phases where the algorithm performs strategy $1$ by mistake. Precisely, using a descending induction we obtain
\begin{equation*}
    \sum_{q\in \Scal^0_p\cap\Pcal^1_p} T^{q+1}_p - T^q_p \leq \frac{T^{q_0}_p-T^{q_0-1}_p}{1-2^{-p}} \leq 2\cdot 2^{-p}T^{q_0}_p \leq 2^{-p+1}T \leq 2\epsilon_p T.
\end{equation*}
On all other phases $\Pcal^0_p\cup(\Pcal^1_p\setminus\Scal^0_p)$, the performance of the learning rule is close to having performed strategy $0$ on all phases. Indeed, using Eq~\eqref{eq:estimate_strategy1_exploitation} we obtain 
\begin{align*}
    \sum_{q\in (\Pcal^1_p\setminus\Scal^0_p)}\tilde R^1_p(q)
    &\geq \sum_{q\in (\Pcal^1_p\setminus\Scal^0_p)}\max_{l=1,\ldots,k(q)}\bar R^l_p(q) -  (T_p^{q+1})^{3/4} - c(T_p^{q+1})^{3/4} \ln T_p^{q+1} - A_p(q)\\
    &\geq  \sum_{q\in (\Pcal^1_p\setminus\Scal^0_p)}\max_{l=1,\ldots,k(q)}\hat R^l_p(q) - \sum_{q<q_0}\left(4(T_p^{q+1})^{7/8}  + c(T_p^{q+1})^{3/4} \ln T_p^{q+1}\right) \\
    &\geq  \sum_{q\in (\Pcal^1_p\setminus\Scal^0_p)}\hat R^0_p(q) -\eta_pT_p^{q_0} - 4(4+c\ln T_p^{q_0})T^{15/16}\\
    &\geq \sum_{q\in (\Pcal^1_p\setminus\Scal^0_p)}\bar R^*_p(q) -\eta_p T - 3\epsilon_p T - 4(5+c\ln T)T^{15/16}.
\end{align*}
In the second inequality, we used Eq~\eqref{eq:exp1_estimates} and in the third inequality, we used the definition of $\Scal^0_p$ and the identities $\sum_{q\leq q_0} (T_p^q)^{7/8} \leq (T_p^{q_0})^{7/8}\frac{2^p}{1-2^{-7/8}}\leq 2^{p+2}(T_p^{q_0})^{7/8} \leq 4 T^{15/16}$. In the last inequality, we used Eq~\eqref{eq:exp0_estimate}. Next, using Eq~\eqref{eq:estimate_strategy0_exploitation} we have directly
\begin{equation*}
    \sum_{q\in\Pcal^0_p} \tilde R^0_p(q) \geq \sum_{q\in\Pcal^0_p}\bar R_p^*(q) - 3\epsilon_p T - 3\cdot 4 T^{15/16}  .
\end{equation*}
Combining the two above inequalities and observing that $\eta_p= 5\epsilon_p$ gives
\begin{align*}
    \sum_{2^{32p}\leq t<T^{q_0}_p, t\in\Tcal_p}&r_t \geq \sum_{q\in \Pcal^0_p}\tilde R^0_p(q) + \sum_{q\in \Pcal_1\setminus\Scal^0}\tilde R^1_p(q)\\
    &\geq \sum_{q\in \Pcal^0_p\cup (\Pcal^1_p\setminus\Scal^0_p)} \bar R_p^*(q) - 11\epsilon_p T -(32+4c\ln T)T^{15/16}\\
    &\geq \sum_{p2^{p+5}\leq q<q_0}\bar R_p^*(q) - \sum_{q\in\Scal^0_p\cap\Pcal^1_p} (T^{q+1}_p - T^q_p) - 11\epsilon_p T -(32+4c\ln T)T^{15/16}\\
    &\geq \sum_{p2^{p+5}\leq q<q_0}\bar R_p^*(q) - 13\epsilon_p T -(32+4c\ln T)T^{15/16}.
\end{align*}
Now recalling the former estimate of $\Rcal_p(T)$ for $T<2^{32p}$, we obtain
\begin{align*}
    \Rcal_p(T) &\geq \Rcal_p(2^{32p}-1) + \sum_{2^{32p}\leq t<T^{q_0}_p, t\in\Tcal_p}r_t\\
    &\geq \bar R^*_p(T)- 2\frac{T}{2^p}  -  \frac{1+c}{2}\sqrt{|\Acal|\ln|\Acal|}T^{1-1/2^7}\log_2 T  - 13\epsilon_p T -(32+4c\ln T)T^{15/16}\\
    &\geq  \bar R^*_p(T) -  \frac{1+c}{2}\sqrt{|\Acal|\ln|\Acal|}T^{1-1/2^7}\log_2 T -(32+4c\ln T)T^{15/16}  - 15\epsilon_p T
\end{align*}
where the term $\frac{T}{2^p}$ comes from the fact that $T-(T^{q_0}_p-1)\leq T^{q_0+1}_p-T^{q_0}_p\leq \frac{T}{2^p}$.
Now note that if $t\in\Tcal_p$, there were at least $4^p$ duplicates, hence $t\geq 4^p$. As a result, we can always suppose without loss of generality that $T\geq 4^p$. Combining with the case $T<2^{32p}$, we obtain that for all $T\geq \max(\hat T_1,\hat T_2)$, $p\geq 0$ with $t\geq 4^p$,
\begin{equation}\label{eq:estimate_on_category}
    \Rcal_p(T) \geq \bar R^*_p(T) - (33+5c) \sqrt{|\Acal|\ln|\Acal|}T^{1-1/2^7}\log_2 T - 15\epsilon_pT.
\end{equation}
This ends the proof that on times $\Tcal_p$, the learning rule has an average error at most $\Ocal(\epsilon_p)$ on the event $\Ecal\cap\Fcal$. Because $\sum_{p\geq 0}\epsilon_p<\infty$, we can afford to converge on each set $\Tcal_p$ to the optimal policy independently. 

Precisely, we aim to show that
\begin{equation*}
    \limsup_{T\to\infty}\frac{1}{T}\sum_{t=1}^T \bar r(\pi^*(X_t), X_t)-r_t\leq 0,\quad (a.s.).
\end{equation*}
Fix $0<\epsilon\leq 1,\delta>0$ and let $p_0$ such that $\sum_{p\geq p_0}\epsilon_p<\frac{\epsilon}{15}$. Because $\Xbb\in\Ccal_2$, by Proposition \ref{prop:C2_equivalent_forms}, $\Xbb^{\leq 4^{p_0}}\in\Ccal_1$. As a result, because the sequence of policies $(\pi^l)_l$ is dense under $\Ccal_1$ processes, there exists $l_0\geq 1$ such that
\begin{equation*}
    \Ebb\left[\limsup_{T\to\infty}\frac{1}{T} \sum_{t\leq T, t\in\Tcal^{\leq 4^{p_0}}} \1[\pi^*(X_t)\neq\pi^{l_0}(X_t)]\right]\leq \frac{\epsilon\delta}{2^{2p_0+2}p_0}.
\end{equation*}
Then, by the dominated convergence theorem, there exists $T_0$ such that
\begin{equation*}
    \Ebb\left[\sup_{T\geq T_0}\frac{1}{T} \sum_{t\leq T, t\in\Tcal^{\leq 4^{p_0}}} \1[\pi^*(X_t)\neq\pi^{l_0}(X_t)]\right]\leq \frac{\epsilon\delta}{2^{2p_0+1}p_0}.
\end{equation*}
In particular, on an event $\Bcal_\delta$ of probability at least $1-\delta$, the Markov inequality yields that for all $T\geq T_0$,
\begin{equation*}
    \sum_{t\leq T, t\in\Tcal^{\leq 4^{p_0}}} \1[\pi^*(X_t)\neq\pi^{l_0}(X_t)] \leq \frac{\epsilon}{2^{2p_0+1}p_0}T.
\end{equation*}
In particular, the above equation holds if we replace $\Tcal^{\leq 4^{p_0}}$ by $\Tcal_p$ for any $p<p_0$.
Now suppose that the event $\Ecal\cap\Fcal\cap \Bcal_\delta$ of probability at least $1-\delta$ is met. For any $p<p_0$ and $q\geq p2^{p+5}$ such that $T^q_p\geq \hat T:=\max( \hat T_1,\hat T_2,2^{l_0},2^{32p_0})$, because $T^q_p\geq 2^{l_0} $, we have
\begin{align*}
    \max_{1\leq l\leq k(q)}\hat R_p^k(q) 
    &\geq \hat R_p^{l_0}(q)\\
    &\geq \bar R^{l_0}_p(q)  -(T_p^{q+1})^{7/8}\\
    &\geq \bar R_p^*(q)  -(T_p^{q+1})^{7/8} - \sum_{t\in\Tcal_p(q)}\1[\pi^*(X_t)\neq\pi^{l_0}(X_t)]
    \\
    &\geq \hat R_p^0(q) - 2(T_p^{q+1})^{7/8} - 3\epsilon_p (T_p^{q+1}-T_p^q) - 2^{-2p-1}T_p^{q+1}\\
    &\geq \hat R_p^0(q) - 2(T_p^{q+1})^{7/8} - 4\epsilon_p (T_p^{q+1}-T_p^q).
\end{align*}
where in the second inequality we used Eq~\eqref{eq:exp1_estimates} and in the fourth we used Eq~\eqref{eq:exp0_estimate}. In the last inequality, we used $2^{-p-1}T_p^{q+1}\leq T_p^{q+1}-T_p^q$. Now let $T_1$ such that $2T^{7/8}< \frac{\epsilon_p}{2^{p+1}}T$ for any $T\geq T_1$. Then, for any $p<p_0$ and $q\geq p2^{p+5}$ such that $T^q_p\geq \tilde T:=\max(\hat T,T_1)$, we have 
\begin{equation*}
    \max_{1\leq l\leq k(q)} \hat R_p^k(q) > \hat R_p^0(q) -5\epsilon_p (T_p^{q+1}-T_p^q),
\end{equation*}
which implies $\Pcal_p(q+1)=1$ since $\eta_p=5\epsilon_p$ if $\Pcal_p(q+1)$ was not already defined. In other terms, starting from time $2^{p_0}\tilde T$, the learning rule always chooses strategy $1$ for categories $p<p_0$. We now bound the error of the learning rule on $\Tcal_p$ for $p<p_0$. Let $\tilde q$ such that $T^{\tilde q-1}_p\leq 2^{p_0}\tilde T<T^{\hat q}_p$. For any $T\geq  2^{p_0}\tilde T$ and $q(T)$ such that $T^{q(T)}_p\leq T<T^{q(T)+1}_p$, we can write
\begin{align*}
    \Rcal_p(T) &- \bar R_p^*(T) \geq  \sum_{\tilde q<q<q(T)} (\tilde R^1_p(q) - \bar R^*_p(q)) + \tilde R^1_p(q(T),T) - \sum_{t\in\Tcal_p(q),t\leq T}\bar r(\pi^*(X_t), X_t)\\
    &\quad\quad\quad\quad\quad - 2^{p_0}\tilde T -\sum_{q<q(T)}A_p(q) \\
    &\geq  \sum_{\tilde q<q<q(T)} ( R^{l_0}_p(q) - \bar R^*_p(q))  - \sum_{t\in\Tcal_p(q),t\leq T}\1[\pi^*(X_t)\neq \pi^{l_0}(X_t)] - 2^{p_0}\tilde T  \\
    &\quad\quad\quad\quad\quad -\sum_{q\leq q(T)}(2A_p(q)+(T_p^{q+1})^{3/4}+c(T_p^{q+1})^{3/4} \ln T_p^{q+1} )\\
    &\geq - \sum_{t\leq T,t\in\Tcal_p} \1[\pi^*(X_t)\neq \pi^{l_0}(X_t)] -2^{p_0}\tilde T-4(3+c)(T_p^{q(T)+1})^{15/16}\ln T_p^{q(T)+1} \\
    &\geq -2^{p_0}\tilde T-16(3+c)T^{15/16}\ln T -\frac{\epsilon}{2^{p_0}p_0}T.
\end{align*}
where in the second inequality we applied Eq~\eqref{eq:estimate_strategy1_exploitation} and Eq~\eqref{eq:exploitation_1_tail}, and in the third inequality, we used the identity $\sum_{q\leq q(T)} (T^{q+1}_p-T^q_p)^{3/4}\leq 4(T^{q(T)+1}_p)^{7/8}$ proved earlier. As a result, we can write
\begin{equation*}
    \sum_{p<p_0} \bar R_p^*(T) - \Rcal_p(T)\leq p_02^{p_0}\tilde T + 16p_0(3+c)T^{15/16}\ln T + \epsilon T.
\end{equation*}
Now because the events $\Ecal,\Fcal$ are met, using Eq~\eqref{eq:estimate_on_category}, we also have for $T\geq \tilde T$
\begin{align*}
    \sum_{p\geq p_0} \bar R_p^*(T) - \Rcal_p(T) &=\sum_{p_0\leq p<\log_{4} T} \bar R_p^*(T) - \Rcal_p(T)\\
    &\leq (17+3c) \sqrt{|\Acal|\ln|\Acal|}T^{1-1/2^7}(\log_2 T)^2 + 15 \sum_{p\geq p_0}\epsilon_p \cdot T\\
    &\leq (17+3c) \sqrt{|\Acal|\ln|\Acal|}T^{1-1/2^7}(\log_2 T)^2 + \epsilon T
\end{align*}
Summing the two above inequalities gives
\begin{multline*}
    \sum_{t=1}^T \bar r(\pi^*(X_t), X_t)-r_t \leq p_02^{p_0}\tilde T + 16p_0(3+c)T^{15/16}\ln T\\
    +(17+3c) \sqrt{|\Acal|\ln|\Acal|}T^{1-1/2^7}(\log_2 T)^2 + 2\epsilon T.
\end{multline*}
As a result, on the event $\Ecal\cap\Fcal\cap\Gcal\cap\Bcal_\delta$ of probability at least $1-\delta$, we have
\begin{equation*}
    \limsup_{T\to\infty}\frac{1}{T} \sum_{t=1}^T \bar r(\pi^*(X_t), X_t)-r_t \leq 2\epsilon.
\end{equation*}
Because this holds for any $\delta>0$ and $0<\epsilon<1$, this shows that almost surely, we have $\limsup_{T\to\infty}\frac{1}{T} \sum_{t=1}^T \bar r(\pi^*(X_t), X_t)-r_t \leq 0.$ We denote by $\Ccal$ this event. We now formally show that the learning rule is universally consistent. Let $\pi:\Xcal\to\Acal$ be a measurable function. First, by the Hoeffding inequality, we have for $T\geq 1$,
\begin{equation*}
    \Pbb\left[\left|\sum_{t=1}^T r_t(\pi(X_t), X_t) - \bar r_t(\pi(X_t), X_t)\right|\leq T^{3/4}\right] 1-e^{-2\sqrt T}.
\end{equation*}
As a result, the Borel-Cantelli lemma implies that on an event $\Hcal$ of probability one, there exists $\hat T_4$ such that for all $T\geq \hat T_4$, $|\sum_{t=1}^T r_t(\pi(X_t), X_t) - \bar r_t(\pi(X_t), X_t)|\leq T^{3/4}$. Then, on $\Ccal\cap\Hcal$ of probability one, for any $T\geq \hat T_4$ we have
\begin{align*}
    \sum_{t=1}^T  r(\pi(X_t), X_t)-r_t &\leq  \sum_{t=1}^T \bar r(\pi(X_t), X_t)-r_t + T^{3/4}\\
    &\leq  \sum_{t=1}^T \bar r(\pi^*(X_t), X_t)-r_t + T^{3/4}.
\end{align*}
Thus, $\limsup_{T\to\infty} \sum_{t=1}^T \bar r(\pi(X_t), X_t)-r_t \leq 0.$ This ends the proof that the learning rule is universally consistent under any $\Ccal_2$ process. Now recall that $\Ccal_2$ is a necessary condition for universal learning by Theorem \ref{thm:SMV-finite-nec}. Hence, the set of learnable processes is exactly $\Ccal =\Ccal_2$ and the learning rule is optimistically universal.
\end{proof}

\section{Countably infinite action space}
\label{sec:countable_actions}

We next turn to the case where the action space is infinite $|\Acal|=\infty$ but countable. The goal of this section is to show that the set of processes admitting universal learning now becomes $\Ccal_1$. This contrasts with the full-feedback setting where universal learning is optimistically achievable under $\Ccal_2$ processes when a property $\ftime$ on the value space $(\Ycal,\ell)$ is satisfied \cite{blanchard:22d}. Intuitively, this asks that mean-estimation is possible in finite time for any prescribed error tolerance. Of interest to the discussion of this section with countable number of actions, \cite{blanchard:22d} showed that countably-infinite classification $(\Ycal,\ell) = (\Nbb,\ell_{01})$ satisfies the $\ftime$ property and, their learning rule is universally consistent under $\Ccal_2$ processes even under noisy and adversarial responses.

For countable action sets, there is a simple optimistically universal learning rule defined as follows. From \citep*[Lemma 24]{hanneke:21}, because $\Acal$ is countable, the $0-1$ loss on $\Acal$ is a separable metric, thus, there exists a countable set $\Pi$ of measurable policies $\pi:\Xcal\to\Acal$ such that for every $\ProcX \in \KC$, 
for every measurable $\goodpol : \X \to \A$, 
\begin{equation*}
\Ebb\left[\inf_{\pi \in \Pi} \hat{\mu}_{\ProcX}( \{ x : \pi(x) \neq \goodpol(x) \} ) \right]\leq \inf_{\pi \in \Pi} \Ebb\left[\hat{\mu}_{\ProcX}( \{ x : \pi(x) \neq \goodpol(x) \} ) \right] = 0,
\end{equation*}
which implies in particular that almost surely, $\inf_{\pi \in \Pi} \hat{\mu}_{\ProcX}( \{ x : \pi(x) \neq \goodpol(x) \} )=0$.
Enumerate $\Pi = \{\pi_1,\pi_2,\ldots\}$.
For any $\ProcX$, we consider the countable set of experts $\{E_1,E_2,\ldots\}$
such that $E_{i,t} = \pi_i(X_t)$.
Our learning rule then applies $\EXPINF$ from Corollary~\ref{cor:infinite-exp4} with this family of experts.

\begin{theorem}
\label{thm:KC-infinite-actions}
Let $\Xcal$ be a separable Borel-metrizable space and $\Acal$ a countable infinite action set. Then, there is an optimistically universal learning rule and the set of learnable processes is $\Ccal =\Ccal_1$.
\end{theorem}
\begin{proof}
We start by showing that the learning rule defined above is universally consistent on any $\Xbb\in\Ccal_1$ process. This proof is essentially identical to that of \citep*[][Theorem 1]{hanneke:22a}. Indeed, denoting by  $\hat a_t$ the action selected by the learning rule at time $t$, Corollary~\ref{cor:infinite-exp4} implies that on an event $\Ecal$ of probability one, for any $\pi\in\Pi$, we have
\begin{equation*}
    \limsup_{T\to\infty}\frac{1}{T}\sum_{t=1}^T r_t(\pi(X_t))-r_t(\hat a_t) \leq 0.
\end{equation*}
Now fix a measurable policy $\pi^\star:\Xcal\to\Acal$. For any $\pi\in\Pi$, because the rewards lie in $[0,1]$, on $\Ecal$,
\begin{align*}
    \limsup_{T\to\infty}\frac{1}{T}&\sum_{t=1}^T r_t(\pi^*(X_t))-r_t(\hat a_t)\\ &\leq \hat \mu_\Xbb(\{x:\pi(x)\neq\pi^*(x)\}) +  \limsup_{T\to\infty}\frac{1}{T}\sum_{t=1}^T r_t(\pi(X_t))-r_t(\hat a_t)\\
    &\leq \hat \mu_\Xbb(\{x:\pi(x)\neq\pi^*(x)\}).
\end{align*}
Also, by construction of the countable set $\Pi$, on an event $\Fcal$ of probability one, we have $\inf_{\pi \in \Pi} \hat{\mu}_{\ProcX}( \{ x : \pi(x) \neq \goodpol(x) \} )=0$. Thus, on $\Ecal\cap\Fcal$, the above inequality shows that $\limsup_{T\to\infty}\frac{1}{T}\sum_{t=1}^T r_t(\pi^*(X_t))-r_t(\hat a_t) \leq 0$. Hence, the learning rule is universally consistent under $\Ccal_1$ processes with adversarial responses.

Next, we show that the condition $\Xbb\in\Ccal_1$ is necessary for the 
existence of a universally consistent learning rule, even for function learning.
Let $\ProcX$ be any process with $\ProcX \notin \KC$.
By Lemma~\ref{lem:infrequent-cells}, there exists a sequence $\{B_i\}_{i=1}^{\infty}$ 
of disjoint measurable subsets of $\X$ 
with $\bigcup\limits_{i \in \nats} B_i = \X$, 
and a sequence $\{N_i\}_{i=1}^{\infty}$ in $\nats$
such that, 
on a $\sigma(\X)$-measurable event $\Event_0$ 
of probability strictly great than zero,
\begin{equation*}
\limsup_{T \to \infty} \frac{1}{T} \sum_{t=1}^{T} \ind\!\left[ \left| \Xbb_{< t} \cap B_{i_t} \right| < N_{i_t} \right] > 0,
\end{equation*}
where $i_t$ is the unique $i \in \nats$ 
with $X_t \in B_i$.

Next, we define the function $\target$.
Enumerate $\A = \{a_1,a_2,\ldots\}$, 
and for each $i \in \nats$, let $A_i = \{a_1,\ldots,a_{2N_i}\}$.
For each $i \in \nats$, let 
$a_i^{\star}$ be an element of $A_i$.
Denote by $\bar{a} = \{a_i^{\star}\}_{i \in \nats}$.
Then for each $i \in \nats$ and each $x \in B_i$, 
define $\target_{\bar{a}}(x,a) = \ind[ a = a_i^{\star} ]$.
Also define $\mathbf{a}_i^{\star}$ 
as $\mathrm{Uniform}(A_i)$ 
(independent over $i$ and all independent of $\ProcX$
and the randomness of the learning rule), 
and $\mathbf{\bar{a}} = \{ \mathbf{a}_{i}^{\star} \}_{i \in \nats}$.
Then for any learning rule $\hat{f}_t$, 
denoting by $\hat{a}_t$ its actions when $\target = \target_{\mathbf{\bar{a}}}$ 
is as constructed above, we have  
\begin{align*}
 \sup_{\bar{a}} \E\! &\left[ \limsup_{T \to \infty} \frac{1}{T} \sum_{t=1}^{T} \left( \sup_{a \in \A} r_t(a) - r_t(\hat{a}_t) \right)  \middle| \mathbf{\bar{a}} = \bar{a} \right]\\
&= \sup_{\bar{a}} \E\!\left[ \limsup_{T \to \infty} \frac{1}{T} \sum_{t=1}^{T} \ind\!\left[ \hat{a}_t \neq \mathbf{a}_{i_t} \right]  \middle| \mathbf{\bar{a}} = \bar{a} \right]
\\ & \geq \E\!\left[ \limsup_{T \to \infty} \frac{1}{T} \sum_{t=1}^{T} \ind\!\left[ \hat{a}_t \neq \mathbf{a}_{i_t} \right] \right]\\
&\geq 
\E\!\left[ \ind_{\Event_0} \cdot \limsup_{T \to \infty} \frac{1}{T} \sum_{t=1}^{T} \ind\!\left[ | \Xbb_{< t} \cap B_{i_t}| < N_{i_t} \right] \ind\!\left[ \hat{a}_t \neq \mathbf{a}_{i_t} \right] \right].
\end{align*}
By the law of total expectation, 
this last expression above equals
\begin{equation*}
\E\!\left[ \ind_{\Event_0} \cdot \E\!\left[ \limsup_{T \to \infty} \frac{1}{T} \sum_{t=1}^{T} \ind\!\left[ | \Xbb_{< t} \cap B_{i_t}| < N_{i_t} \right] \ind\!\left[ \hat{a}_t \neq \mathbf{a}_{i_t} \right] \middle| \ProcX, \hat{f}_{\cdot} \right] \right],
\end{equation*}
where conditioning on $\hat{f}_{\cdot}$ 
indicates we condition on the independent 
randomness of the learning rule.
Since the average is bounded for any fixed $T$, 
Fatou's lemma, together with the 
fact that 
$\ind\!\left[ | \Xbb_{< t} \cap B_{i_t}| < N_{i_t} \right]$ 
is $\sigma(\ProcX)$-measurable, 
imply the expression above is at least as large as 
\begin{equation}
\label{eqn:countable-lb-intermediate-1}
\E\!\left[ \ind_{\Event_0} \cdot \limsup_{T \to \infty} \frac{1}{T} \sum_{t=1}^{T} \ind\!\left[ | \Xbb_{< t} \cap B_{i_t}| < N_{i_t} \right] \P\!\left( \hat{a}_t \neq \mathbf{a}_{i_t} \middle| \ProcX, \hat{f}_{\cdot} \right) \right].
\end{equation}

Let $\hat{N}_{t} = | \Xbb_{\leq t} \cap B_{i_t} |$ and 
$\hat{A}_{t} = \{ \hat{a}_{t'} : t' \leq t, i_{t'} = i_{t} \}$.
Note that, conditioned on 
$\hat{f}_{\cdot}$ and $\ProcX$, 
the probability that 
$\mathbf{a}_{i_t} \in \hat{A}_{t}$ is 
at most $\hat{N}_{t} \frac{1}{|A_{i_t}|} = \frac{\hat{N}_{t}}{2N_{i_t}}$.
In particular, this implies that 
if $\hat{N}_t \leq N_{i_t}$, 
the conditional probability (given $\hat{f}_{\cdot}$ and $\ProcX$)  that $\hat{a}_{t} \neq \mathbf{a}_{i_t}$ 
is at least $1 - \frac{\hat{N}_{t}}{2N_{i_t}} \geq \frac{1}{2}$.
Thus, \eqref{eqn:countable-lb-intermediate-1} 
is no smaller than 
\begin{equation}
\label{eqn:countable-lb-intermediate-2}
\E\!\left[ \ind_{\Event_0} \cdot \limsup_{T \to \infty} \frac{1}{T} \sum_{t=1}^{T} \ind\!\left[ | \Xbb_{< t} \cap B_{i_t}| < N_{i_t} \right] \cdot \frac{1}{2} \right].
\end{equation}
By definition of the event $\Event_0$, 
there is a nonzero probability that 
\begin{equation*}
\ind_{\Event_0} \cdot \limsup_{T \to \infty} \frac{1}{T} \sum_{t=1}^{T} \ind\!\left[ | \Xbb_{< t} \cap B_{i_t}| < N_{i_t} \right] > 0,
\end{equation*}
and since the quantity on the left hand size is non-negative, 
this further implies the expectation in \eqref{eqn:countable-lb-intermediate-2}
is also strictly greater than zero.

Altogether, this implies there exists a (non-random) choice of $\bar{a}$ 
such that, choosing $\target = \target_{\bar{a}}$, 
the actions $\hat{a}_t$ made by the 
learning rule $\hat{f}_{t}$ satisfy 
\begin{equation*}
\E\!\left[ \limsup_{T \to \infty} \frac{1}{T} \sum_{t=1}^{T} \left( \sup_{a \in \A} r_{t}(a) - r_{t}(\hat{a}) \right) \right] > 0,
\end{equation*}
and since the quantity in the expectation 
is non-negative, this further implies 
that for this choice of $\target$, 
with non-zero probability, 
\begin{equation*}
\limsup_{T \to \infty} \frac{1}{T} \sum_{t=1}^{T} \left( \sup_{a \in \A} r_{t}(a) - r_{t}(\hat{a}) \right) > 0.
\end{equation*}
Thus, $\hat{f}_t$ is not 
universally consistent for 
function learning.  Since this holds 
for any choice of learning rule $\hat{f}_{\cdot}$, this completes the proof.
\end{proof}

\section{Uncountable action space}
\label{sec:uncountable_actions}

We next consider the case of uncountable action spaces. In this section, we assume that $\Acal$ is an uncountable separable Borel metrizable space. In this case, we will show that universal consistency is impossible even in the simplest setting where rewards are a deterministic, i.e., $r_t(a) = f^*(X_t,a)$ for some unknown measurable function $f^*:\Xcal\times\Acal\to [0,1]$. The argument is based on a dichotomy depending whether there exists a non-atomic probability measure $\mu$ on $\Acal$, i.e., such that for all $a\in\Acal$, we have $\mu(\{x\})=0$. If this is not the case, we will need the following simple result which states that any stochastic process $\Xbb$ takes values in a countable set $Supp(\Xbb)$ almost surely.

\begin{lemma}
\label{lemma:supports}
Let $\Xcal$ a metrizable separable Borel space such that there does not exist a non-atomic probability measure on $\Xcal$. Then, for any random variable $X$ on $\Xcal$ there exists a countable set $Supp(X)\subset \Xcal$ such that almost surely, $X\in Supp(X)$. Similarly, for any stochastic process $\Xbb$ on $\Xcal$ there exists a countable set $Supp(\Xbb)\subset \Xcal$ such that almost surely $\forall t\geq 1, X_t\in Supp(\Xbb)$.
\end{lemma}

\begin{proof}
Fix $\Xcal$ such a space and let $X$ be a random variable on $\Xcal$. Let $Supp(X) = \{x\in\Xcal: \Pbb[X=x]>0\}$. Suppose by contradiction that $\Pbb[X\notin Supp(X)]>0$ and denote $\Ecal$ the corresponding event. Because $\Pbb[\Ecal]>0$ we can consider a random variable $Y\sim X|\Ecal$. For instance take $(X_i)_{i\geq 1}$ an i.i.d. process following the distribution of $X$, fix $x_0\in\Xcal$ a fixed arbitrary instance, and pose
\begin{equation*}
    Y = \begin{cases}
        X_{\hat k} & \text{if } \{i \geq 1: X_i\notin Supp(X)\}\neq \emptyset, \quad \hat k = \min \{i \geq 1: X_i\notin Supp(X)\},\\
        x_0 &\text{otherwise.}
    \end{cases}
\end{equation*}
Because the first time $k$ such that $X_k\notin Supp(X)$ is a geometric variable $\Gcal(1-\Pbb[\Ecal])$, the event $\Fcal = \{\exists i \geq 1: X_i\notin Supp(X)\}$ has probability one. We now show that $Y$ is non-atomic. First observe that $Y\notin Supp(X)$. Then, if $x\in \Xcal\notin Supp(X)$, we have
\begin{equation*}
    \Pbb[Y = x] = \Pbb[\{Y = x\}\cap\Fcal] = \Pbb[\{X_{\hat k} = x\}\cap\Fcal] \leq \Pbb\left[\bigcup_{i\geq 1}\{X_i=x\}\right] \leq \sum_{i\geq 1}\Pbb[X_i=x] = 0.
\end{equation*}
where in the first equality we used the fact that $\Pbb[\Fcal^c] = 0$. Therefore $Y$ is non-atomic which contradicts the hypothesis on $\Xcal$. As a result, almost surely $X\in Supp(X)$. It now suffices to check that $Supp(X)$ is countable, which is guaranteed by the identity $1 = \Pbb[X\in Supp(X)] =\sum_{x\in Supp(X)} \Pbb[X=x]$, since each term of the sum is positive. This ends the proof of the first claim.

Now let $\Xbb$ be a stochastic process on $\Xcal$ and define $Supp(\Xbb) = \bigcup_{t\geq 1} Supp(X_t)$. Then $Supp(\Xbb)$ is countable as countable union of countable sets and
\begin{equation*}
    \Pbb[\exists t\geq 1: X_t\notin Supp(\Xbb)] \leq \sum_{t\geq 1} \Pbb[X_t\notin Supp(\Xbb)] \leq \sum_{t\geq 1} \Pbb[X_t\notin Supp(X_t)] = 0.
\end{equation*}
This ends the proof of the lemma.
\end{proof}

We are now ready to show that no process admits universal learning when the action set is uncountable.

\begin{theorem}
\label{thm:uncountable_emptyset}
If $\A$ is an uncountable separable Borel metrizable space, then there does not exist any $\ProcX$ admitting universal consistency for measurable function learning.
\end{theorem}

\begin{proof}
Fix a learning rule $f_\cdot$ and for any $a^*\in\Acal$, we define the reward function $f_{a^*}^*(x,a) = \1[a=a^*]$ for $x\in\Xcal,a\in\Acal$. We also define the policy $\pi_{a^*}:x\in\Xcal\mapsto a^*\in\Acal$. We first consider the case where there exists a non-atomic probability measure $\mu$ on $\Acal$. Then, for any $t\geq 1$, and consider the case where $a^*$ is sampled from the distribution $\mu$ independently from the process $\Xbb$ and the randomness of the learning rule. Then we have
\begin{equation*}
    \Pbb_{a^*\sim \mu}[f_t(\Xbb_{<t},(0)_{<t},X_t) = a^*] = \Ebb_{\Xbb,f_t}[\Pbb_{a^*\sim\mu}(f_t(\Xbb_{<t},(0)_{<t},X_t) = a^*)]=0.
\end{equation*}
Denote by $\Ecal_t$ this event. Then, by the union bound, $\Pbb[\bigcap_{t\geq 1}\Ecal_t]=1$. The law of total probability implies that there exists a deterministic choice of $a^*$ such that
\begin{equation*}
    \Pbb[\forall t\geq 1,f_t(\Xbb_{<t},(0)_{<t},X_t) \neq a^*]=1,
\end{equation*}
where the probability is taken over $\Xbb$ and the randomness of the learning rule.

Now suppose that there does not exist non-atomic probability measures on $\Acal$. From Lemma \ref{lemma:supports}, for any probability measure $\mu$ on $\Acal$, we can construct a countable set $Supp(\mu)\subset\Xcal$ such that $\mu(Supp(\mu))=1$. Now consider the set
\begin{equation*}
    S = \bigcup_{t\geq 1}  Supp(f_t(\Xbb_{\leq t-1}, (0)_{\leq t-1}, X_t)).
\end{equation*}
Then, $S$ is countable as the union of countable sets. Since $\Acal$ is uncountable, let $a^*\in \Acal\setminus S$. By construction, on an event of probability one, for all $t\geq 1$, we have $f_t(\Xbb_{\leq t-1}, (0)_{\leq t-1}, X_t)\neq a^*$.

In both cases, we found an action $a^*\in\Acal$ such that on an event $\Ecal$ of probability one over $\Xbb$ and the randomness of the learning rule, having received 0 reward in the past history at time step $t$, the learning rule does not select $a^*$, hence receives reward 0 at time $t$ as well. Thus, by induction, denoting by $\hat a_t$ the action selected by the learning rule at time $t$ for reward $f_{a^*}^*$, we have $ \Ecal\subset\{\forall t\geq 1, \hat a_t\neq a^*\}.$ Thus, on $\Ecal$,
\begin{equation*}
    \limsup_{T\to\infty}\frac{1}{T}\sum_{t=1}^T r_t(\pi^*(X_t))-r_t(\hat a_t) = 1.
\end{equation*}
Because $\Ecal$ has probability one, this shows that $f_\cdot$ is not universally consistent.
\end{proof}

\section{Universal learning under continuity assumptions}
\label{sec:continuity-assumptions}

In Section \ref{sec:uncountable_actions} we showed that for general uncountable separable metric actions spaces, without further assumptions on the rewards, one cannot achieve universal consistency. The goal of this section is to show that adding mild continuity assumptions on the rewards enables to significantly enlarge the set of processes admitting universal learning.

\subsection{Continuous rewards}
\label{sec:continuous_rewards}

In this section, we suppose that the rewards are continuous as defined in Definition \ref{def:continuous+unif_cont_rewards}, and show that universal consistency on $\Ccal_1$ processes is still achievable. For bounded separable metric action spaces $(\tilde \Acal,\tilde d)$, \cite{hanneke:21} showed that there is countable set of measurable policies $\Pi$ such that for any measurable $\pi^*:\Xcal\to\tilde \Acal$ and $\Xbb\in\Ccal_1$,
\begin{equation*}
    \inf_{\pi\in\Pi} \Ebb\left[\limsup_{T\to\infty}\frac{1}{T}\sum_{t=1}^T \tilde d(\pi^*(X_t),\pi(X_t))\right] =0.
\end{equation*}
In general, the action space $(\Acal,d)$ is unbounded, however, $(\Acal,d\wedge 1)$ is a separable bounded metric space on which we can apply the above result. This provides a countable set of measurable policies $\Pi$ such that for any measurable $\pi^*:\Xcal\to\Acal$ and $\Xbb\in\Ccal_1$,
\begin{equation*}
    \inf_{\pi\in\Pi} \Ebb\left[\limsup_{T\to\infty}\frac{1}{T}\sum_{t=1}^T \tilde d(\pi^*(X_t),\pi(X_t)) \wedge 1\right] =0.
\end{equation*}
From this observation, we can get the following lemma.

\begin{lemma}\label{lemma:density_continuous_rewards}
    Let $\Xcal$ be a separable metrizable Borel space and $(\Acal,d)$ be a separable metric space. For any measurable function $\pi^*:\Xcal\to\Acal$, on an event of probability one, for all $i\geq 1$, there exists $\pi^i\in \Pi$ such that
    \begin{equation*}
        \limsup_{T\to\infty}\frac{1}{T} \sum_{t=1}^T \1[d(\pi^*(X_t),\pi^i(X_t)) \geq 2^{-i}] \leq 2^{-i},
    \end{equation*}
    for all $i\geq 1$, $\frac{1}{T}\sum_{t\leq T} r_t(\pi^i(X_t))-\bar r_t(\pi^i(X_t)) \to 0$ and similarly for $\pi^*$.
\end{lemma}

\begin{proof}
By construction of the countable set of policies $\Pi$, for any $i\geq 1$, there exists $\pi^i\in \Pi$ such that
\begin{equation*}
    \Ebb\left[\limsup_{T\to\infty}\frac{1}{T}\sum_{t=1}^T d(\pi^*(X_t),\pi^i(X_t))\wedge 1\right] \leq 2^{-3i}.
\end{equation*}
Then, Markov's inequality implies that with probability at least $1-2^{-i}$.
\begin{equation*}
    \limsup_{T\to\infty}\frac{1}{T}\sum_{t=1}^T d(\pi^*(X_t),\pi^i(X_t))\wedge 1 \leq 2^{-2i}.
\end{equation*}
Applying Markov's inequality a second time, we obtain
\begin{equation*}
    \limsup_{T\to\infty}\frac{1}{T} \sum_{t=1}^T \1[d(\pi^*(X_t),\pi^i(X_t)) \geq 2^{-i}] \leq 2^i \limsup_{T\to\infty}\frac{1}{T} \sum_{t=1}^T d(\pi^*(X_t),\pi^i(X_t))\wedge 1  \leq 2^{-i}.
\end{equation*}
The Borel-Cantelli lemma implies that on an event $\Ecal$ of probability one, for $i$ sufficiently large, there exists $\pi^i\in\Pi$ with $\limsup_{T\to\infty}\frac{1}{T} \sum_{t=1}^T \1[d(\pi^*(X_t),\pi^i(X_t)) \geq 2^{-i}] \leq 2^{-i}$. Clearly, this implies that this is the case for all $i\geq 1$. For any $i\geq 1$, Azuma's inequality implies that with probability at least $1-4e^{-2i\sqrt T}$, we have
\begin{equation*}
    \left|\sum_{t=1}^T r_t(\pi^i(X_t))-\bar r_t(\pi^i(X_t))\right|, \left|\sum_{t=1}^T r_t(\pi^*(X_t))-\bar r_t(\pi^*(X_t))\right| \leq  2i T^{3/4}.
\end{equation*}
Because $\sum_{T\geq 1}\sum_{i\geq 1}e^{-2i\sqrt T}<\infty$, the Borel-Cantelli lemma implies that on an event $\Fcal$ of probability one, for all $i\geq 1$, $\frac{1}{T}\sum_{t\leq T}r_t(\pi^i(X_t))-\bar r_t(\pi^i(X_t))\to 0$ and similarly for $\pi^*$. Therefore, on the event $\Ecal\cap\Fcal$ of probability one, all events are satisfied, which ends the proof of the lemma.
\end{proof}

Using Lemma \ref{lemma:density_continuous_rewards}, we will show that the $\EXPINF$ algorithm over the set of policies $\Pi$ is optimistically universal for continuous rewards.

\begin{theorem}\label{thm:continuous_rewards_C1}
Let $(\Acal,d)$ be an infinite separable metric space. Then, $\EXPINF$ is optimistically univesal for continuous rewards and the set of learnable processes for continuous rewards is $\Ccal ^{c} = \Ccal_1$.
\end{theorem}

\begin{proof}
We start by showing that $\EXPINF$ is universally consistent under continuous rewards under $\Ccal_1$ processes. Let $\Xbb\in\Ccal_1$ and continuous rewards $(r_t)_t$ and let $\pi^*:\Xcal\to\Acal$ be measurable policy. We denote $\Ecal$ the event on which the guarantee for $\EXPINF$ of Corollary \ref{cor:infinite-exp4} holds. For convenience, we also note $\hat a_t$ the action selected by the learning rule at time $t$. For any $x\in\Xcal$, and $\epsilon>0$, we define
\begin{equation*}
    \Delta_\epsilon(x) = \sup_{a\in\Acal: d(a,\pi^*(x))\leq \epsilon} |\bar r(a, x)- \bar r(\pi^*(x), x)|.
\end{equation*}
Next, fix $\delta>0$, and for any $\epsilon>0$, let $A(\epsilon,\delta) = \{x\in\Xcal: \Delta_\epsilon(x) \geq \delta\}.$ Note that for any $x\in\Xcal$, by continuity of $\bar r(\cdot, x)$, for any $\delta>0$, $\bigcap_{\epsilon>0} A(\epsilon,\delta)=\emptyset$. By Lemma \ref{lemma:density_continuous_rewards}, on an event $\Fcal$ of probability one, for any $i\geq 1$, there exists $\pi^i\in\Pi$ such that
\begin{equation*}
    \limsup_{T\to\infty}\frac{1}{T}\sum_{t=1}^T\1[d(\pi^*(X_t),\pi^i(X_t)\geq 2^{-i}] \leq 2^{-i},
\end{equation*}
$\frac{1}{T}\sum_{t\leq T} r_t(\pi^i(X_t)-\bar r_t(\pi^i(X_t))\to 0$ and similarly for $\pi^*$. As a result, on $\Fcal$, for any $i\geq 1$,
\begin{align*}
    &\limsup_{T\to\infty}\frac{1}{T} \sum_{t\leq T}  r_t(\pi^*(X_t), X_t)-  r_t(\pi^i(X_t), X_t)\\
    &\leq \hat \mu_{\Xbb}(A(2^{-i},\delta)) + 2^{-i} + \limsup_{T\to\infty}\frac{1}{T} \sum_{\substack{t\leq T,\\d(\pi^*(X_t),\pi^i(X_t))<2^{-i}\\
    \Delta_{2^{-i}}(X_t)<\delta}} \bar r_t(\pi^*(X_t), X_t)- \bar r_t(\pi^i(X_t), X_t)\\
    &\leq \hat \mu_{\Xbb}(A(2^{-i},\delta)) + 2^{-i} + \delta.
\end{align*}
Because $\Xbb\in\Ccal_1$ and $A(2^{-i},\delta)\downarrow\emptyset$, on an event $\Gcal(\delta)$ of probability one, we have that $\hat \mu_{\Xbb}(A(2^{-i},\delta)) \underset{i\to\infty}{\longrightarrow} 0$. Last, let $\delta_j=2^{-j}$ for any $j\geq 0$. On the event $\Ecal\cap\Fcal\cap\bigcap_{j\geq 0}\Gcal(\delta_j)$ of probability one, combining Corollary \ref{cor:infinite-exp4} together with the above inequality implies that for any $j\geq 0$,
\begin{equation*}
    \limsup_{T\to\infty}\frac{1}{T} \sum_{t\leq T}  r_t(\pi^*(X_t), X_t)-  r_t(\hat a_t, X_t) \leq \delta_j.
\end{equation*}
Thus, $\limsup_{T\to\infty}\frac{1}{T} \sum_{t\leq T}  r_t(\pi^*(X_t), X_t)-  r_t(\hat a_t, X_t)\leq 0\;(a.s.)$, which shows that $\EXPINF$ is universally consistent under $\Xbb$ for stationary rewards. This ends the proof of the theorem.

We now show that $\Ccal_1$ is necessary for universal consistency. The proof is analogous to that of Theorem \ref{thm:KC-infinite-actions} in which we proved that for unrestricted rewards on countably infinite action sets, $\Ccal_1$ is necessary for universal learning. Suppose that $\Xbb\in\Ccal_1$ and let $f_\cdot$ be a learning rule. Using the same arguments, there exist a partition of $\Xcal$ in measurable sets $\{B_i\}_{i\geq 1}$  and a sequence $\{N_i\}_{i\geq 1}$ of integers such that with non-zero probability,
\begin{equation*}
    \limsup_{T\to\infty} \frac{1}{T}\sum_{t=1}^T \1[|\Xbb_{<t}\cap B_{i_t}|<N_{i_t}] >0,
\end{equation*}
where $i_t$ is the index such that $X_t\in B_i$. As in the original proof, let $\{a_i,i\geq 1\}$ be a sequence of distinct actions and let $A_i=\{a_1,\ldots,a_{2N_i}\}$ for $i\geq 1$. We also define $\epsilon_i = \min_{a\neq a' \in A_i} d(a,a')$ the minimum distance within $A_i$ actions. For any sequence $\bar a=\{a_i^*\}_{i\in\Nbb}$ where $a_i^*\in A_i$ for $i\geq 1$, we define a deterministic reward $r^*_{\bar a}$ with
\begin{equation*}
    r^*_{\bar a}(a, x) = \max\left(1-\frac{2 d(a,a^*_i)}{\epsilon_i},0\right),
\end{equation*}
for any $x\in B_i$, which defines a proper measurable continuous reward. We also define the rewards $\tilde r^*_{\bar a}(a, x) = \1[a=a^*_i]$ for $a\in\Acal$ and $x\in B_i$. We now define the learning rule $\tilde f_\cdot$ which at each step $t$ computes the action $\hat a$ chosen by the learning rule $f_\cdot$, selects the action $\tilde a_t := \argmin_{a'\in A_i}d(\hat a,a')$ where $i\geq 1$ is the unique index with $X_t\in B_i$, receives a reward $r_t$, then reports the reward $\max\left(1-\frac{2 d(\hat a,\tilde a)}{\epsilon_i},0\right)$, which will be then used by $f_\cdot$ for future action selections. Note that on $B_i$, the rewards $r^*_{\bar a}$ were defined so that they are identically zero outside of the balls $B_d(a,\epsilon_i)$ for $a\in A_i$. These are disjoint, so the report of reward given by $\tilde f_\cdot$ to its internal run of $f_\cdot$ coincides exactly with what $f_\cdot$ would have received by selecting action $\hat a$ instead of $\tilde a$. Further, one can observe that selecting one of the nearest element within $A_i$ always increases the reward because the balls $B_d(a,\epsilon_i)$ for $a\in A_i$ are disjoint. Therefore, $\tilde f_\cdot$ always receives higher reward than $f_\cdot$ at any step. Now observe that $\tilde f_\cdot$ always observes a reward in $\{0,1\}$. Hence, for any choice of $\bar a$, at any step $t$, $\tilde f_t$ has the same rewards on $r^*_{\bar a}$ as it would have obtained on the rewards $\tilde r^*_{\bar a}$. Therefore,
\begin{align*}
    \limsup_{T \to \infty} \frac{1}{T} \sum_{t=1}^{T} \left( \sup_{a \in \A}  r^*_{\bar a,t}(a) -  r^*_{\bar a,t}(\hat a_t) \right) &\geq \limsup_{T \to \infty} \frac{1}{T} \sum_{t=1}^{T} \left( \sup_{a \in \A}  r^*_{\bar a,t}(a) -  r^*_{\bar a,t}(\tilde  a_t) \right)\\
&= \limsup_{T \to \infty} \frac{1}{T} \sum_{t=1}^{T} \left( \sup_{a \in \A} \tilde r^*_{\bar a,t}(a) - \tilde r^*_{\bar a,t}(\tilde a_t) \right),
\end{align*}
where $\hat a_t$ (resp. $\tilde a_t$) denotes the action selected by $f_\cdot$ (resp. $\tilde f_\cdot$) at time $t$. However, the proof of Theorem \ref{thm:KC-infinite-actions} precisely shows that there exists a choice of $\bar a$ such that with non-zero probability, $\limsup_{T \to \infty} \frac{1}{T} \sum_{t=1}^{T} \left( \sup_{a \in \A} \tilde r^*_{\bar a,t}(a) - \tilde r^*_{\bar a,t}(\tilde a_t) \right) >0$. Now observe that the measurable function $\pi(x) = a_i^*$ where $x\in B_i$ always selects the best action. This show that $f_\cdot$ is not consistent on rewards $r^*_{\bar a}$, hence not universally consistent. This shows that $\Xbb\notin\Ccal ^c$ and completes the proof of the theorem.
\end{proof}

\subsection{Uniformly-continuous rewards}
\label{sec:uniformly_continuous_rewards}
In the last section, we showed that adding a continuity constraint on the rewards allowed to learn $\Ccal_1$ processes even when the action space $\Acal$ is infinite. Unfortunately, this additional assumption on the rewards is not sufficient to obtain universal consistency on the more general class of processes $\Ccal_2$. In this section, we strengthen the assumptions on the rewards and suppose that they are  uniformly-continuous in the actions as per Definition \ref{def:continuous+unif_cont_rewards}.

We start by giving necessary conditions for uniformly-continuous rewards. To do so, we will need the following simple reduction, showing that some necessary conditions provided in the unrestricted rewards case can be used in the uniformly-continuous setting as well.

\begin{lemma}\label{lemma:reduction}
    Let $(\Acal,d)$ be a separable metric space. Let $S\subset \Acal$ such that we have $\min_{a,a'\in S}d(a,a')>0$. Then, $\Ccal ^{uc}(\Acal)\subset \Ccal (S)$.
\end{lemma}

\begin{proof}
Intuitively, we restrict the problem on $\Acal$ to the actions $S$. Formally, let $\eta = \frac{1}{3}\min_{a,a'\in S}d(a,a')$ and observe that any reward function $r:S\to [0,\bar r]$ can be extended to a  uniformly-continuous function $F(r):\Xcal\to\Acal$ as follows.
\begin{equation*}
    F(r)(a) = \max\left(0,\max_{a'\in S}r(a')-d(a,a') \frac{\bar r}{\eta}\right),\quad a \in \Acal.
\end{equation*}
Note that this function is $\frac{\bar r}{\eta}-$Lipschitz, hence  uniformly-continuous---in the case where rewards are stochastic, we can still apply this transformation at the realization-level. Further, the sets $B_d(a',\eta)$ for $a'\in S$ are all disjoint by triangular inequality. Thus, for all $a'\in S$, we have $F(r)(a')=r(a')$. We now describe the reduction from uniformly-continuous rewards on $\Acal$ to unrestricted rewards on $S$. Let $\Xbb\in\Ccal(\Acal)$ and we denote by $\hat a_t$ the action selected at time $t$ by an universally consistent learner $f_\cdot$ under $\Xbb$ for uniformly-continuous rewards on $\Acal$. We now construct a learning rule for unrestricted rewards on $S$. First, for $a\in\Acal$, denote by $NN_S(a) =\argmin_{a'\in S} d(a,a')$ the index of the nearest neighbor of $a$ in $S$ where ties are broken arbitrarily, e.g., by lexicographic order (necessarily, $S$ is countable because $\Acal$ is separable). We consider the learning rule which selects the actions $NN_S(\hat a_t)$, i.e.,
\begin{equation*}
    f_t^{S}(\mb x_{\leq t-1},\mb r_{\leq t-1},x_t) = NN_S(f_t(\mb x_{\leq t-1},\mb r_{\leq t-1},x_t))
\end{equation*}
for all $x_{\leq t}\in\Xcal^t$ and $r_{\leq t-1}\in [0,\bar r]^{t-1}$. We aim to show that $f^S_\cdot$ is universally consistent under $\Xbb$ for unrestricted rewards on $S$. Fix any reward mechanism $r$ on the action space $S$. We consider the reward mechanism $\tilde r$ on the action space $\Acal$ as follows,
\begin{equation*}
    \tilde r_t(a, x) = F(r(\cdot\mid x))(a),
\end{equation*}
for any $a\in \Acal$. Note that the mechanism $\tilde r$ only depends on the nearest neighbor of selected actions. Denote $\tilde a_t$ the corresponding selected action. Observe that by construction of the functional $F$, for any $t\geq 1$, $\tilde r_t(\tilde a_t)\geq \tilde r_t(\hat a_t)$. Thus, by monotonicity, $f_\cdot^S$ is also consistent on reward mechanism $\tilde r$. Now note that $\tilde f_\cdot$ only selects actions within $S$ and receives the same rewards that would have been observed by running the learning rule on reward mechanism $r$. As a result, $f^{S}_\cdot$ is also consistent for reward $r$. This ends the proof that it is universally consistent under $\Xbb$ and hence $\Xbb\in\Ccal(S)$. This ends the proof of the proposition.

\comment{

\begin{equation*}
    \tilde r_t(a\mid \mb x_{\leq t}, \mb a_{\leq t-1})= F(r(\cdot\mid \mb x_{\leq t},(NN_S(a_s))_{s\leq t-1})).
\end{equation*}
Note that $NN$ takes values in $S$ hence $(NN_S(a_s))_{s\leq t-1}$ are valid actions that the mechanism $r_t$ could possibly take into account. Note that the mechanism $(\tilde r_t)_t$ only depends on the nearest neighbor of selected actions. Hence, the learning trajectory observed by using the learning rule $f_\cdot$ is equivalent to that of the learning rule $f^S_\cdot$ using rewards $(\tilde r_t)_t$. Denote $\tilde a_t$ the corresponding selected action. Observe that by construction of the functional $F$, for any learning trajectory and any $t\geq 1$, $\tilde r_t(\tilde a_t)\geq \tilde r_t(\hat a_t)$. Thus, by monotonicity, $f_\cdot^S$ is also consistent on rewards $(\tilde r_t)_t$. Now note that $\tilde f_\cdot$ only selects actions within $S$ and receives the same rewards that would have been observed by running the learning rule on reward mechanism $(r_t)_t$. As a result, $f^{S}_\cdot$ is also consistent on rewards $(r_t)_t$. This ends the proof that it is universally consistent under $\Xbb$ and hence $\Xbb\in\Ccal(S)$. This ends the proof of the proposition.

}
\end{proof}

As a direct consequence of Lemma \ref{lemma:reduction} and the results from previous sections, we can use the necessary conditions from the unrestricted reward setting by changing the terms ``finite action set'' (resp. ``countably infinite action set'') into ``totally-bounded action set'' (resp. ``non-totally-bounded action set'').

\begin{corollary}\label{cor:simple_upper_bounds}
Let $\Acal$ be a non-totally-bounded metric space. Then, $\Ccal^{uc} \subset \Ccal_1$.
Let $\Acal$ be a totally-bounded metric space with $|\Acal|>2$. Then, $\Ccal^{uc}  \subset \Ccal_2$.
\end{corollary}

We now turn to sufficient conditions and show that we can recover the results from the unrestricted case as well. For non-totally-bounded value spaces, the $\EXPINF$ learning rule from Theorem \ref{thm:continuous_rewards_C1} is already universally consistent under $\Ccal_1$ processes, which is a necessary condition by Corollary \ref{cor:simple_upper_bounds}. As a result, imposing the uniformly-continuous assumption on the rewards does not improve the set of learnable processes.

\begin{theorem}
    Let $\Xcal$ be a separable Borel metrizable space and $\Acal$ a non-totally-bounded metric space. Then, $\Ccal ^{uc}=\Ccal_1$.
\end{theorem}

Next, we consider totally-bounded actions spaces and generalize the learning rule for stochastic rewards in finite action spaces. Recall that this learning rule associates to each time a category $p=\textsc{Category}(t)$, based on the number of previous occurrences of $X_t$, and works separately on each category. Within each category, the algorithm balances between two strategies: strategy 0 which uses independent $\EXP$ learners for each distinct instance, and strategy 1 which performs $\EXPINF$. We adapt the algorithm in the following way. First, the $\EXP$ learners from strategy 0 search for the best action within $\Acal(\delta_p)$, an $\delta_p-$net of $\Acal$ where $\delta_p$ will be defined carefully. Note that since $\Acal$ is possibly infinite, restricting strategy 0 to finite action sets is necessary. However, we aim for arbitrary precision, hence we will have $\delta_p\to 0$ as $p\to\infty$. Second, for strategy 1, we use the countable set of functions $\Pi$ defined as for the $\EXPINF$ algorithm in Theorem \ref{thm:continuous_rewards_C1}.

\begin{theorem}
Let $\Acal$ be a totally-bounded metric space. Then, there exists an optimistically universal learning rule for stationary and uniformly-continuous rewards, and learnable processes are $\Ccal ^{uc}=\Ccal_2$.
\end{theorem}

\begin{proof}
We first define the new learning rule. \textsc{Category} and \textsc{AssignPurpose} are left unchanged. We will use the countable set of policies $\Pi=\{\pi^l,l\geq 1\}$ as in the continuous case in Lemma \ref{lemma:density_continuous_rewards}, for $\textsc{Explore}(1;\cdot)$, and Algorithm \ref{alg:main_learning_rule}. Further, in $\textsc{Explore}(0;\cdot)$ and Algorithm \ref{alg:main_learning_rule}, $\EXP_\Acal$ is replaced by $\EXP_{\Acal(\delta_p)}$. Finally, in \textsc{SelectStrategy}, $\eta_p = 10\frac{\sqrt{|\Acal|\ln |\Acal|}}{2^{p/4}}$ is replaced by $\eta_p = 10\frac{\sqrt{|\Acal(\delta_p)|\ln |\Acal(\delta_p)|}}{2^{p/4}}$, where we will define $\delta_p$ shortly. In the original proof of the universal consistence of the algorithm, we showed that the average error of the learning rule on category $p$, $\Tcal_p$ is $\Ocal(\tilde \epsilon_p)$ where $\tilde \epsilon_p=2\frac{\sqrt{|\Acal|\ln|\Acal|}}{2^{p/4}}$. Similarly, we now define $\epsilon_p:=2\frac{\sqrt{|\Acal(\delta_p)|\ln |\Acal(\delta_p)|}}{2^{p/4}}$. A key feature of the proof is that since we had $\sum_p \tilde \epsilon_p<\infty$, the learner can afford to converge on each set $\Tcal_p$ separately. We mimic this behavior by choosing $\delta_p$ such that $\sum_p \epsilon_p<\infty$. Precisely, we pose
\begin{equation*}
    \delta_p = \min\{2^{-i}: |\Acal(2^{-i})|\ln|\Acal(2^{-i})|\leq 2^{p/4}\}.
\end{equation*}
As a result, we obtain directly $\epsilon_p\leq 2^{-1-p/8}$ which is summable, and $\delta_p\to 0$ as $p\to\infty$.

We now show that this learning rule is universally consistent under processes $\Xbb\in\Ccal_2$ by adapting the proof of Theorem \ref{thm:opt_rule_stat}. Fix $r$ a reward mechanism. For every $\epsilon>0$, there exists $\Delta(\epsilon)$ such that
\begin{equation*}
    \forall x\in\Xcal,\forall a,a'\in\Acal,\quad d(a,a')\leq \Delta(\epsilon)\Rightarrow |\bar r(a, x)-\bar r(a', x)| \leq \epsilon.
\end{equation*}
For every $\delta>0$, we will also define $\epsilon(\delta)=2\inf \{\epsilon>0: \Delta(\epsilon)\geq \delta\}$. By uniform-continuity, $\epsilon(\delta)\to 0$ as $\delta\to 0$ and because of the factor $2$, we have
\begin{equation*}
    \forall x\in\Xcal,\forall a,a'\in\Acal,\quad d(a,a')\leq \delta \Rightarrow |\bar r(a, x)-\bar r(a', x)| \leq \epsilon(\delta).
\end{equation*}

Now observe that in the original proof, the probabilistic bounds $p_i(p,q)$ for $1\leq i\leq 8$ do not depend on the cardinality of the action set. Therefore, on the same event $\Ecal\cap\Fcal$ of probability one, Eq~\eqref{eq:exp0_estimate}, \eqref{eq:estimate_strategy0_exploitation}, \eqref{eq:exp1_estimates}, \eqref{eq:estimate_strategy1_exploitation}, \eqref{eq:exploitation_1_tail} and \eqref{eq:exploration_bound} hold starting from some time $\hat T$, for the intended values of $p,q,T$.
The only difference, however, is that in strategy 0, we perform $\EXP$ over the restricted action set $\Acal(\delta_p)$. As a result, for any $x\in\Xcal$, we have
\begin{equation*}
    \max_{a\in\Acal(\delta_p)}\bar r(a, x) \geq \max_{a\in\Acal}\bar r(a, x) - \epsilon(\delta_p).
\end{equation*}
As a result, Eq~\eqref{eq:exp0_estimate} should be replaced with
\begin{align*}
    \hat R_p^0(q) &\geq   \bar R_p^*(q)- (T_p^{q+1})^{\frac{7}{8}} - 6\frac{\sqrt{ |\Acal|\ln |\Acal|}}{2^{p/4}}(T_p^{q+1}-T_p^q) -\epsilon(\delta_p)|\Tcal_p(q)|\\
    \hat R_p^0(q) &\leq   \bar R_p^*(q)- (T_p^{q+1})^{\frac{7}{8}} - 6\frac{\sqrt{ |\Acal|\ln |\Acal|}}{2^{p/4}}(T_p^{q+1}-T_p^q).
\end{align*}
Note that the additional term $\epsilon(\delta_p)|\Tcal_p(q)|$ is not present in the upper bound because searching over $\Acal$ (in $\bar R_p^*(q)$) is always better than searching over $\Acal(\delta_p)$ (in $\hat R_p^0(q)$). Similarly, Eq~\eqref{eq:estimate_strategy0_exploitation} should be replaced with
\begin{equation*}
    \tilde R^0_p(q) \geq \bar R_p^*(q) - 6\frac{\sqrt{ |\Acal|\ln |\Acal|}}{2^{p/4}}(T_p^{q+1}-T_p^q) - (T_p^{q+1})^{3/4} - A_p(q) - \epsilon(\delta_p)|\Tcal_p(q)|.
\end{equation*}
Similarly, the adapted Eq~\eqref{eq:performance_initial_phase} becomes
\begin{equation*}
    \Rcal_p(T) \geq \bar R^*_p(T) -  \frac{1+c}{2}\sqrt{|\Acal|\ln|\Acal|}T^{1-1/2^7}\log_2 T -\frac{T}{2^p}-\epsilon(\delta_p)|\Tcal_p\cap\{t\leq T\}|.
\end{equation*}
Furthering the same bounds, Eq~\eqref{eq:estimate_on_category} becomes
\begin{equation*}
    \Rcal_p(T) \geq \bar R^*_p(T) - (33+5c) \sqrt{|\Acal|\ln|\Acal|}T^{1-1/2^7}\log_2 T - 15\epsilon_pT - 2\epsilon(\delta_p)|\Tcal_p\cap\{t\leq T\}|.
\end{equation*}

We are now ready to prove universal consistence of our learning rule. Fix $0<\epsilon<1$, and as in the original proof, let $p_0$ such that $\sum_{p\geq p_0}\epsilon_p<\frac{\epsilon}{15}$, because $\sum_p\epsilon_p<\infty$. Again, we have $\Xbb^{\leq 4^{p_0}} \in\Ccal_1$ and as a result, we can apply Lemma \ref{lemma:density_continuous_rewards}. As a result, on an event $\Hcal$ of probability one, for all $\epsilon>0$, there exists $i(\epsilon)\geq 1$ such that $2^{-i(\epsilon)} \leq \Delta(\epsilon),\epsilon$ and $\pi^{i(\epsilon)}\in\Pi$ such that
\begin{align*}
    \limsup_{T\to\infty}&\frac{1}{T} \sum_{t\leq T,t\in\Tcal^{\leq 4^{p_0}}} \bar r_t(\pi^*(X_t)) - \bar r_t(\pi^i(X_t)) \\
    &\leq \limsup_{T\to\infty}\frac{1}{T} \sum_{t\leq T,t\in\Tcal^{\leq 4^{p_0}}} \1[d(\pi^*(X_t),\pi^{i(\epsilon)}(X_t) \geq 2^{-i(\epsilon)}] \\
    &+ \limsup_{T\to\infty}\frac{1}{T} \sum_{t\leq T,t\in\Tcal^{\leq 4^{p_0}}} (\bar r_t(\pi^*(X_t)) - \bar r_t(\pi^{i(\epsilon)}(X_t)))\1[d(\pi^*(X_t),\pi^{i(\epsilon)}(X_t) \leq \Delta(\epsilon)]\\
    &\leq 2^{-i(\epsilon)} + \epsilon \leq 2\epsilon,
\end{align*}
where $\pi^*$ denotes the optimal policy. We define the events $\Ecal,\Fcal$ as in the original proof. In the rest of the proof, we will now suppose that the event $\Ecal\cap\Fcal\cap\Hcal$ of probability one is satisfied. On this event, because the parameter $\epsilon>0$ was arbitrary in the above derivations, gthere exists $l_0\geq 1$ (random index) such that 
\begin{equation*}
    \limsup_{T\to\infty} \frac{1}{T}\sum_{t\leq T,t\in\Tcal^{\leq 4^{p_0}}} \bar r_t(\pi^*(X_t)) - \bar r_t(\pi^{l_0}(X_t)) \leq \frac{\epsilon}{2^{2p_0+2}}.
\end{equation*}

Following the same arguments as in the original proof, for $p<p_0$, and $T_p^q$ sufficiently large, we need to adapt the following estimates.
\begin{align*}
    \max_{1\leq l\leq k(q)} &\hat R_p^k(q) 
    \geq \hat R_p^{l_0}(q)\\
    &\geq \bar R^{l_0}_p(q)  -(T_p^{q+1})^{7/8}\\
    &\geq \bar R_p^*(q)  -(T_p^{q+1})^{7/8} - \sum_{t\in\Tcal_p(q)}(r_t(\pi^*(X_t)) - \bar r_t(\pi^{l_0}(X_t)))
    \\
    &\geq \hat R_p^0(q) - 2(T_p^{q+1})^{7/8} - 3\epsilon_p (T_p^{q+1}-T_p^q) - \sum_{t\in\Tcal_p(q)}r_t(\pi^*(X_t)) - \bar r_t(\pi^{l_0}(X_t)).
\end{align*}
Then, observe that
\begin{equation*}
    \limsup_{q\to\infty} \frac{2(T_p^{q+1})^{7/8} + 3\epsilon_p (T_p^{q+1}-T_p^q) + \sum_{t\in\Tcal_p(q)}r_t(\pi^*(X_t)) - \bar r_t(\pi^{l_0}(X_t))}{T_p^{q+1}-T_p^q} \leq 4\epsilon_p<\eta_p.
\end{equation*}
Thus, as in the original proof, starting from some time $\tilde T$, the learning rule always chooses strategy 1 over strategy 0 for all categories $p\leq p_0$.

We continue the same arguments to obtain for $p<p_0$ and $T\geq 2^{p_0}\tilde T$,
\begin{equation*}
    \Rcal_p(T)-\bar R^*_p(T) \geq -2^{p_0}\tilde T-16(3+c)T^{15/16}\ln T - \sum_{t\leq T, t\in\Tcal_p} \bar r_t(\pi^*(X_t)) - \bar r_t(\pi^{l_0}(X_t)),
\end{equation*}
which yields
\begin{equation*}
    \sum_{p<p_0} \bar R_p^*(T) - \Rcal_p(T)\leq p_02^{p_0}\tilde T + 16p_0(3+c)T^{15/16}\ln T +  \sum_{t\leq T} \bar r_t(\pi^*(X_t)) - \bar r_t(\pi^{l_0}(X_t)).
\end{equation*}
Noting that $\limsup_{T\to\infty}\frac{1}{T}\sum_{t\leq T} \bar r_t(\pi^*(X_t)) - \bar r_t(\pi^{l_0}(X_t)) \leq \epsilon$, from there, the same arguments show that the learning rule is universally consistent.
\end{proof}

As a summary, with the uniform-continuity assumption we could generalize all results from the unrestricted rewards case with the corresponding totally-bounded/non-totally-bounded dichotomy on action spaces.

\section{Unbounded rewards}
\label{sec:unbounded}

In this section, we allow for unbounded rewards $\Rcal=[0,\infty)$ and start with the unrestricted rewards setting---no continuity assumption. Recall that in this setting, we assume that for any context $x\in\Xcal$ and action $a\in\Acal$, the random variable $r(a, x)$ is integrable so that the immediate expected reward is well defined. 

When $\Acal$ is uncountable, we showed that even for bounded rewards, no process $\Xbb$ admits universal learning. Therefore, we will focus on the case when $\A$ is 
finite or countably infinite, and show that
$\FS$ determines whether universal 
consistency is possible. Moreover, a simple variant of $\EXPINF$ suffices for optimistically universal learning as follows. Enumerate $\A = \{a_1,a_2,\ldots, a_{|\A|}\}$
(or $\A = \{a_1,a_2,\ldots\}$ for countably 
infinite $\A$) and for any observed instance $x\in\Xcal$, we run an independent $\EXPINF$ where the experts of the sequence are the constant policies equal to $a_i$ for $1\leq i \leq |\Acal|$, i.e., the expert $E_i$ always selects action $a_i$.

\begin{theorem}\label{thm:unbounded_rewards}
Let $\Acal$ be a countable action set with $|\Acal|\geq 2$. Then, there is an optimistically universal learning rule and the set of learnable processes admitting universal is $ \Ccal_3$.
\end{theorem}

The fact that $\Ccal_3$ characterizes universal learning was already the case in the noiseless full-feedback setting \cite{blanchard:22b}, hence Theorem \ref{thm:unbounded_rewards} shows that for unrestricted rewards, we can achieve universal learning in the partial feedback setting without generalization cost.

\begin{proof}
    First, even in the full-information feedback setting, $\Xbb\in\Ccal_3$ is known to be necessary for universal consistency \citep*{blanchard:22b}. A fortiori in the bandit setting, this condition is still necessary $\Ccal\subset\Ccal_3$.
    
    We now show that the learning rule defined above is universally consistent under $\Ccal_3$ processes. For simplicity, we denote by $\hat a_t$ the action selected be the learning rule at time $t$. Fix $\Xbb\in\Ccal_3$ and define $S = \{x\in\Xcal: \Xbb\cap\{x\}\neq\emptyset\}$ the support of the process. By definition of $\Ccal_3$, almost surely, $|S|<\infty$. We denote by $\Ecal$ this event of probability one. Next, for any $x\in S$, we define $\Tcal(x) = \{t: X_t=x\}$ and let $\tilde S = \{x\in S: |\Tcal(x)|=\infty\}$ the set of points which are visited an infinite number of times. Recall that the learning rule performs an independent $\EXPINF$ subroutine on the times $\Tcal(x)$ for all $x\in S$. As a result, by Corollary \ref{cor:infinite-exp4}, for any $x\in \tilde S$, with probability one, for all $a\in \Acal$,
    \begin{equation*}
        \limsup_{T\to\infty}\frac{1}{|\Tcal(x)\cap\{t\leq T\}|}\sum_{t\in\Tcal(x),t\leq T} r_t(a) - r_t(\hat a_t) \leq 0.
    \end{equation*}
    Now observe that $\tilde S$ is countable. Hence, by the union bound, on an event $\Fcal$ of probability one, for all $x\in \tilde S$ and $a\in \Acal$, we have
    \begin{equation*}
        \limsup_{T\to\infty}\frac{1}{T}\sum_{t\leq T, t\in\Tcal(x)} r_t(a) - r_t(\hat a_t) \leq \limsup_{T\to\infty}\frac{1}{|\Tcal(x)\cap\{t\leq T\}|}\sum_{t\leq T, t\in\Tcal(x)} r_t(a) - r_t(\hat a_t) \leq 0.
    \end{equation*}
    In the rest of the proof, we suppose that $\Ecal\cap\Fcal$ is met. On $\Ecal$, there exists $\hat T = 1+\max \{t: X_t=x,x\in S\setminus \tilde S\}$ such that for any $T\geq \hat T$, we have $X_t\in\tilde S$. Then, for any policy $\pi^*:\Xcal\to\Acal$, and $T\geq 1$, we have
    \begin{equation*}
        \sum_{t=1}^T r_t(\pi^*(X_t))-r_t(\hat a_t) \leq \sum_{t\leq \hat T} r_t(\pi^*(X_t)) + \sum_{x\in \tilde S} \sum_{t\leq T, t\in\Tcal(x)} r_t(a) - r_t(\hat a_t).
    \end{equation*}
    As a result, because $\Fcal$ is met,
    \begin{equation*}
        \limsup_{T\to\infty} \frac{1}{T}\sum_{t=1}^T r_t(\pi^*(X_t))-r_t(\hat a_t) \leq \sum_{x\in \tilde S} \limsup_{T\to\infty} \frac{1}{T} \sum_{t\leq T, t\in\Tcal(x)} r_t(a) - r_t(\hat a_t) \leq 0.
    \end{equation*}
    using the fact that $\Pbb[\Ecal\cap\Fcal]=1$, we proved that the learning rule is universally consistent under any $\Ccal_3$ process. This ends the proof of the theorem.
\end{proof}

The last remaining question is whether this very restrictive set of processes $\Ccal_3$ can be improved under the continuity and uniform-continuity assumptions from Definition \ref{def:continuous+unif_cont_rewards}.

Unfortunately, we show that this is not the case for continuous rewards, however, the continuity assumption allows to achieve universal consistence on $\Ccal_3$ processes even on uncountable action spaces. Recall that by Theorem \ref{thm:uncountable_emptyset}, universal consistency was not achievable for uncountable spaces in the unrestricted reward case.

\begin{theorem}\label{thm:continuous_unbounded_rewards}
    Let $\Xcal$ be a separable metrizable Borel space and $(\Acal,d)$ be a separable metric space with $|\Acal|\geq 2$. Then, there is an optimistically universal learning rule for continuous unbounded rewards and the set of learnable processes for universal learning with continuous unbounded rewards is $\Ccal_3$.
\end{theorem}

\begin{proof}
In the case of countable action set $\Acal$ with $|\Acal|\geq 2$, Theorem \ref{thm:unbounded_rewards} already showed that $\Ccal_3$ is sufficient for universal learning under continuous unbounded rewards. Therefore, it remains to show that in the case of uncountable action space, $\Ccal_3$ is still sufficient for universal learning. More precisely, we will show that the same learning rule which assigns a distinct $\EXPINF$ learner to each distinct instance of $\Xbb$ as defined in Theorem \ref{thm:unbounded_rewards} is still universally consistent under $\Ccal_3$ processes. The only difference is that we run the learners $\EXPINF$ on a dense sequence of actions $(a_i)_{i\geq 1}$ of the complete action set $\Acal$ which may be uncountable. Let $\Xbb\in\Ccal_3$. We use the same notations as in the original proof of Theorem \ref{thm:unbounded_rewards} for the support $S=\{x\in\Xcal:\Xbb\cap\{x\}\neq\emptyset\}$, the event $\Ecal = \{|S|<\infty\}$, $\Tcal(x) = \{t:X_t=x\}$ for $x\in S$ and $\tilde S = \{x\in S:|\Tcal(x)|=\infty\}$. By Corollary \ref{cor:infinite-exp4}, for any $x\in\tilde S$, with probability one, for all $i\geq 1$, we have now
\begin{equation*}
    \limsup_{T\to\infty}\frac{1}{|\Tcal(x)\cap\{t\leq T\}|}\sum_{t\in\Tcal(x),t\leq T} r_t(a_i) - r_t(\hat a_t) \leq 0.
\end{equation*}
Let $a\in\Acal$ and $\epsilon>0$, because $(a_i)_{i\geq 1}$ is dense in $\Acal$ and the immediate reward is continuous, there exists $i(\epsilon)$ such that $|\bar r(a_{i(\epsilon)}) - \bar r(a)|\leq \epsilon$. Now observe that by the union bound, for any $x\in\tilde S$, with probability one, by the law of large numbers one has for all $i\geq 1$,
\begin{equation*}
    \frac{1}{|\Tcal(x)\cap\{t\leq T\}|}\sum_{t\in\Tcal(x),t\leq T} r_t(a_i) \underset{T\to\infty}{\longrightarrow} \bar r_t(a_i),
\end{equation*}
and similarly for $a$. As a result, for any $x\in\tilde S$, with probability one, for any $\epsilon>0$,
\begin{align*}
    &\limsup_{T\to\infty}\frac{1}{|\Tcal(x)\cap\{t\leq T\}|}\sum_{t\in\Tcal(x),t\leq T} r_t(a) - r_t(\hat a_t) \\
    &\leq \bar r(a) - \bar r(a_{i(\epsilon)}) + \limsup_{T\to\infty}\frac{1}{|\Tcal(x)\cap\{t\leq T\}|}\sum_{t\in\Tcal(x),t\leq T} r_t(a_{i(\epsilon)}) - r_t(\hat a_t)\\
    &\leq \epsilon.
\end{align*}
As a result, we showed that for any $x\in\tilde S$, and any $a\in\Acal$, with probability one,
\begin{equation*}
    \limsup_{T\to\infty}\frac{1}{|\Tcal(x)\cap\{t\leq T\}|}\sum_{t\in\Tcal(x),t\leq T} r_t(a) - r_t(\hat a_t) \leq 0.
\end{equation*}
Now fix $\pi^*:\Xcal\to\Acal$ a measurable policy. Because $\tilde S$ is countable, by the union bound, on an event $\Fcal$ of probability one, for all $x\in\tilde S$, we have
\begin{align*}
    &\limsup_{T\to\infty}\frac{1}{T}\sum_{t\leq T, t\in\Tcal(x)} r_t(\pi^*(x)) - r_t(\hat a_t) \\
    &\leq \limsup_{T\to\infty}\frac{1}{|\Tcal(x)\cap\{t\leq T\}|}\sum_{t\leq T, t\in\Tcal(x)} r_t(\pi^*(x)) - r_t(\hat a_t) \leq 0.
\end{align*}
Then, the same arguments as in the original proof show that on $\Ecal\cap\Fcal$, for any $T\geq 1$, one has
\begin{equation*}
    \limsup_{T\to\infty} \frac{1}{T}\sum_{t=1}^T r_t(\pi^*(X_t))-r_t(\hat a_t) \leq \sum_{x\in \tilde S} \limsup_{T\to\infty} \frac{1}{T}\sum_{t\leq T, t\in\Tcal(x)} r_t(a) - r_t(\hat a_t) \leq 0.
\end{equation*}
Thus, the learning rule is universally consistent under $\Ccal_3$ processes.

We now show that $\Xbb\in \Ccal_3$ is still necessary for universal learning with continuous rewards. For the unrestricted reward case, this was a direct consequence of a result of \citep*{blanchard:22b}, which we now adapt for continuous rewards. First, for any $\Xbb\notin\Ccal_3$, they show that there exists a disjoint measurable partition $\{B_i\}_{i=1}^\infty$ such that with non-zero probability, $|\{i:\Xbb\cap B_i\neq \emptyset\}|=\infty$ on an event $\Ecal_0$. Then, they constructed a sequence of times $T_i$ for $i\geq 1$ such that on an event $\Ecal$ of probability one, for sufficiently large indices $i$, $\tau_i:=\min\{0\}\cup\{t: X_t\in B_i\} \leq T_i$. Now fix two distinct actions $a_0,a_1\in\Acal$, let $\epsilon = \frac{d(a_0,a_1)}{3}$ and fix a learning rule $f_\cdot$. We denote by $\hat a_t$ its selected action at time $t$. Consider the following rewards
\begin{equation}
    r^{\mb U}(a, x) = \max\left(0,T_i\left(1-\frac{d(a,a_{U_j})}{\epsilon}\right) \right),\quad x\in B_j,
\end{equation}
for any binary sequence $\mb U$. Now suppose that they were sampled from an i.i.d. sequence of Bernouillis $\Bcal(\frac{1}{2})$, independent of the process $\Xbb$ and the randomness of the learning rule. Now observe that for any $i\geq 1$ such that $\tau_i\leq T_i$, with probability at least $\frac{1}{2}$ independently of the past, we have $\hat a_{\tau_i}\notin B(a_{U_j},\epsilon)$, which implies $\max_{a\in\Acal}r^{\mb U}_{\tau_i}(a) - r^{\mb U}_{\tau_i}(\hat a_{\tau_i}) \geq T_i$. From there, the same arguments as in the original proof show that with probability one, this event occurs infinitely often and $\Ecal$ is met, which by the law of total probability implies that there exists a deterministic choice of values for $\mb U=(U_j)_{j\geq 1}$ such that on the corresponding deterministic (hence stationary) rewards, the learning rule is not consistent on $\Ecal_0\cap\Ecal$ which has non-zero probability. This shows that $\Xbb$ does not admit universal learning even in the simplest case of deterministic continuous rewards.
\end{proof}

Last, we investigate the case of uniformly-continuous unrestricted rewards. Unfortunately, the uniform continuity assumption over the immediate expected rewards does not provide any advantage over the continuity assumption.

\begin{proposition}\label{prop:uniformly_continuous_unbounded}
Let $\Xcal$ be a separable metrizable Borel space and $\Acal$ be a separable metric space with $|\Acal|\geq 2$. Then, the set of learnable processes for universal learning with uniformly-continuous unbounded rewards is $\Ccal_3$.
\end{proposition}
\begin{proof}
It suffices to show that the $\Ccal_3$ condition is still necessary for universal learning under uniformly-continuous rewards since the sufficiency is guaranteed by Theorem \ref{thm:continuous_unbounded_rewards}. We adapt the proof of the necessity of $\Ccal_3$ in the continuous unbounded reward case. Let $\Xbb\notin\Ccal_3$ and suppose that there exists an universally consistent learning rule $f_\cdot$ under $\Xbb$ for uniformly-continuous unbounded rewards. We use the same notations as in the proof of Theorem \ref{thm:continuous_unbounded_rewards}. We now define a sequence $(M_i)_{i\geq 1}$ recursively such that $M_1 = 2T_1$ and for any $i\geq 1$, $M_{i+1} = 2T_{i+1} + 4T_{i+1} \sum_{j\leq i} M_j$. Then, consider the following stochastic rewards
\begin{equation*}
    r(a, x) = \begin{cases}
        M_i\left(1+  \frac{d(a,a_0)\wedge d(a_0,a_1)}{d(a_0,a_1)}\right)  &\text{w.p.}\; \frac{1}{2},\\
        M_i\left(1- \frac{d(a,a_0)\wedge d(a_0,a_1)}{d(a_0,a_1)}\right) &\text{w.p.}\; \frac{1}{2}.
    \end{cases}\quad x\in B_i,i\geq 1.
\end{equation*}
These rewards are uniformly-continuous because for any $x\in \Xcal$, the expected immediate reward is $\bar r(a, x) = 0$ for all $a\in\Acal$. Now for $u\in\{0,1\}$, define the constant policy $\pi^u:x\in\Xcal\mapsto a_u\in\Acal$. Denote by $\hat a_t$ the action selected by the learning rule at time $t$. Because it is consistent under the rewards mechanism given by $r$, using $\pi^0$, $\pi^1$ and the union bound, we have that almost surely, for any $u\in\{0,1\}$,
\begin{equation}\label{eq:two_almost_sure_events}
    \limsup_{T\to\infty} \frac{1}{T}\sum_{t=1}^T r_t(a_u, X_t) - r_t(\hat a_t, X_t) \leq 0.
\end{equation}
Now recall that on the event $\Ecal_0$ of non-zero probability, we have $|\{i:\Xbb\cap B_i\neq\emptyset\}|=\infty$. In other terms, $|\{i:\tau_i>0\}|=\infty$. We then define the random sequence of indices $(i_k)_{k\geq 1}$ such that on $\Ecal_0^c$, $i_k=0$ for all $k\geq 1$ and on $\Ecal_0$, the indices are defined recursively such that $i_1 = \argmin_{i\geq 1, \tau_i>0} \tau_i$ and for $k\geq 1$, we have $i_{k+1} = \argmin_{i>i_k, \tau_i>0} \tau_i$. The $\argmin$ are well defined because on $\Ecal_0$, all the times $\tau_i$ for $i\in\{j\geq 1:\tau_j>0\}$ are distinct. As a result, by construction of the recursion, on $\Ecal_0$, the sequence $(i_k)_{k\geq 1}$ is an increasing sequence of times and for all $k\geq 1$, we have
\begin{equation*}
    \{i: \Xbb_{<\tau_{i_k}}\cap B_i\neq\emptyset\} = \{i: 0<\tau_i<\tau_{i_k}\} \subset \{1\leq i< i_k\}.
\end{equation*}
Now recall that on the event $\Ecal$ of probability one, there exists $\hat i\geq 1$ such that for any $i\geq \hat i$, we have $\tau_i:=\min\{0\}\cup\{t:X_t\in B_i\} \leq T_i$. Therefore, on $\Ecal_0\cap\Ecal$, letting $\hat k=\min\{k:i_k\geq\hat i\}$, we have that for $k\geq \hat k$, and $u\in\{0,1\}$
\begin{align*}
    \sum_{t=1}^{\tau_{i_k}-1} r_t(a_u, X_t) - r_t(\hat a_t, X_t) &\geq \sum_{i: \Xbb_{<\tau_{i_k}}\cap B_i\neq\emptyset}\sum_{t<\tau_{i_k},X_t\in B_i} (-2M_i)\\
    &\geq -2\sum_{i<i_k} T_{i_k} M_i \\
    &\geq -\frac{M_{i_k}}{2} + T_{i_k}.
\end{align*}
Now observe that on the event $\Ecal_0\cap\Ecal$ which has non-zero probability, if $d(\hat a_{\tau_{i_k}},a_0) \geq \frac{d(a_0,a_1)}{2}$ and the reward on $B_{i_k}$ at time $\tau_{i_k}$ is in its negative alternative, i.e., $r(a, x) = M_i\left(1- \frac{d(a,a_0)\wedge d(a_0,a_1)}{d(a_0,a_1)}\right)$, we have
\begin{equation*}
    \frac{1}{\tau_{i_k}}\sum_{t=1}^{\tau_{i_k}} r_t(a_0, X_t) - r_t(\hat a_t, X_t) \geq \frac{1}{\tau_{i_k}}\left(\frac{M_{i_k}}{2}-\frac{M_{i_k}}{2} + T_{i_k}\right) \geq 1.
\end{equation*}
Now by construction, the negative alternative occurs with probability $\frac{1}{2}$, independently from the past history and the complete process $\Xbb$. As a result, for any $k\geq 1$, we have
\begin{equation}\label{eq:bad_event1}
    \Pbb\left[\frac{1}{\tau_{i_k}}\sum_{t=1}^{\tau_{i_k}} r_t(a_0, X_t) - r_t(\hat a_t, X_t) \geq 1 \mid \Ecal_0,\Ecal, k\geq \hat k,d(\hat a_{\tau_{i_k}},a_0) \geq \frac{d(a_0,a_1)}{2} \right] \geq \frac{1}{2}.
\end{equation}
Similarly, one can check that on the event $\Ecal_0\cap\Ecal$, if $d(\hat a_{\tau_{i_k}},a_0) < \frac{d(a_0,a_1)}{2}$ and the reward on $B_{i_k}$ at time $\tau_{i_k}$ is in its positive alternative, we have
\begin{equation*}
    \frac{1}{\tau_{i_k}}\sum_{t=1}^{\tau_{i_k}} r_t(a_1, X_t) - r_t(\hat a_t, X_t) \geq \frac{1}{\tau_{i_k}}\left(\frac{M_i}{2}-\frac{M_{i_k}}{2} + T_{i_k} \right) \geq 1.
\end{equation*}
As a result, the same arguments as above give
\begin{equation}\label{eq:bad_event2}
    \Pbb \left[\frac{1}{\tau_{i_k}}\sum_{t=1}^{\tau_{i_k}} r_t(a_1, X_t) - r_t(\hat a_t, X_t) \geq 1 \mid \Ecal_0,\Ecal, k\geq \hat k,d(\hat a_{\tau_{i_k}},a_0) < \frac{d(a_0,a_1)}{2} \right] \geq \frac{1}{2}.
\end{equation}
Finally, define for any $T\geq 1$ the event
\begin{equation*}
    \Fcal_T = \left\{ \frac{1}{T}\sum_{t=1}^T r_t(a_0, X_t) -r_t(\hat a_t, X_t) \geq 1\right\} \cup \left\{\frac{1}{T}\sum_{t=1}^{T} r_t(a_1, X_t) - r_t(\hat a_t, X_t) \geq 1\right\}.
\end{equation*}
We obtain for any $k\geq 1$,
\begin{align*}
    &\Pbb[\Fcal_{\tau_{i_k}}\mid \Ecal_0,\Ecal, k\geq \hat k]\\
    &\geq \Pbb\left[\Fcal_{\tau_{i_k}}\mid \Ecal_0,\Ecal, k\geq \hat k, d(\hat a_{\tau_{i_k}},a_0) \geq \frac{d(a_0,a_1)}{2}\right] \Pbb\left[ d(\hat a_{\tau_{i_k}},a_0) \geq \frac{d(a_0,a_1)}{2}\mid \Ecal_0,\Ecal, k\geq \hat k\right]\\
    &+ \Pbb\left[\Fcal_{\tau_{i_k}}\mid \Ecal_0,\Ecal, k\geq \hat k, d(\hat a_{\tau_{i_k}},a_0) < \frac{d(a_0,a_1)}{2}\right] \Pbb\left[ d(\hat a_{\tau_{i_k}},a_0) < \frac{d(a_0,a_1)}{2}\mid \Ecal_0,\Ecal, k\geq \hat k\right]\\
    &\geq \frac{1}{2} \Pbb\left[ d(\hat a_{\tau_{i_k}},a_0) \geq \frac{d(a_0,a_1)}{2}\mid \Ecal_0,\Ecal, k\geq \hat k\right] +  \frac{1}{2}\Pbb\left[ d(\hat a_{\tau_{i_k}},a_0) < \frac{d(a_0,a_1)}{2}\mid \Ecal_0,\Ecal, k\geq \hat k\right]\\
    &=\frac{1}{2},
\end{align*}
where in the second inequality we used Eq~\eqref{eq:bad_event1} and Eq~\eqref{eq:bad_event2}. As a result, using Fatou's lemma
\begin{align*}
    \Pbb[\Fcal_{\tau_{i_k}} \text{occurs for infinitely many }k\geq 1\mid \Ecal_0,\Ecal] &\geq \limsup_{k\geq 1}\Pbb[\Fcal_{\tau_{i_k}}\mid \Ecal_0,\Ecal] \\
    &\geq \frac{1}{2}\limsup_{k\geq 1}\Pbb[k\geq \hat k\mid \Ecal_0,\Ecal] = \frac{1}{2},
\end{align*}
where in the last inequality, we used the dominated convergence theorem given that on the event $\Ecal$, $\hat k<\infty$. As a result, we showed that
\begin{equation*}
    \Pbb\left[\exists u\in\{0,1\}, \limsup_{T\to\infty}\frac{1}{T}\sum_{t=1}^T r_t(a_u, X_t) - r_t(\hat a_t, X_t) \geq 1 \mid \Ecal_0,\Ecal\right] \geq \frac{1}{2}.
\end{equation*}
However, because $\Pbb[\Ecal\cap\Ecal_0]=\Pbb[\Ecal_0]>0$, Eq~\eqref{eq:two_almost_sure_events} shows that 
\begin{equation*}
    \Pbb\left[\forall u\in\{0,1\},\limsup_{T\to\infty}\frac{1}{T}\sum_{t=1}^T r_t(a_u, X_t) - r_t(\hat a_t, X_t) \geq 1 \mid \Ecal_0,\Ecal\right] =1,
\end{equation*}
which contradicts the previous inequality. This shows that the learning rule was not consistent under the rewards $(r_t)_t$, hence not universally consistent under $\Xbb$. This shows that $\Ccal_3$ is necessary for universal learning and completes the proof.
\end{proof}

\begin{funding}
This work is being partly funded by ONR grant N00014-18-1-2122.
\end{funding}



\bibliographystyle{imsart-number} 
\bibliography{refs}       


\end{document}